\documentclass[runningheads, letterpaper, 11pt]{article} 
\usepackage{times,fullpage}
\usepackage{url}
\usepackage{hyperref}
\usepackage{graphicx,color}
\usepackage{array,float}
\usepackage{amstext,amssymb,amsmath}
\usepackage{amsthm, verbatim, bm}
\usepackage{algorithmic}
\usepackage{algorithm}
\usepackage{pdfsync}
\usepackage{geometry}
\usepackage{subfigure}

\DeclareMathOperator*{\argmin}{argmin}

\DeclareMathOperator{\pdf}{pdf}
\def\final{0}
\ifnum\final=0
\newcommand{\mynote}[3]{\marginpar{\parbox{0.7in}{\tiny {\color{#2} {\sc #1}: {\sf #3}}}}}
\else
\newcommand{\mynote}[3]{}
\fi

\newtheorem{lemma} {Lemma}

\newtheorem{thm}{Theorem}

\newtheorem{defn}{Definition}
\newtheorem{corollary}{Corollary}

\newcommand{\E}{\mathbb{E}}
\newcommand{\I}{\mathbb{I}}

\newcommand{\D}{\mathcal{D}}
\newcommand{\C}{\mathcal{C}}
\newcommand{\R}{\mathcal{R}}
\newcommand{\A}{\mathcal{A}}

\newcommand{\supp}{\text{supp}}

\newcommand{\grad}{\bigtriangledown}

\newcommand{\ltwo}[1]{||#1||_2}

\newcommand{\x}{\bm{x}}
\renewcommand{\b}{\bm{b}}
\newcommand{\ocp}{{\sf OCP}}
\newcommand{\giga}{{\sf GIGA}}
\newcommand{\pgiga}{{\sf PGIGA}}

\newcommand{\pol}{{\sf POL}}
\newcommand{\ftl}{{\sf FTL}}
\newcommand{\qftl}{{\sf QFTL}}
\newcommand{\pqftl}{{\sf PQFTL}}
\newcommand{\igd}{{\sf IGD}}
\newcommand{\pigd}{{\sf PIGD}}
\newcommand{\pocp}{{\sf POCP}}
\newcommand{\hatx}{\hat\x}

\newcommand{\tx}{\tilde{\x}}
\newcommand{\xt}{\x_{t+1}}
\newcommand{\w}{\bm{w}}
\newcommand{\txt}{\tx_{t+1}}
\newcommand{\dxt}{\Delta\x_{t+1}}
\newcommand{\G}{{\mathcal G}}
\newcommand{\ie}{i.e., }

\newcommand{\B}{\bm{B}}

\newcommand{\wu}{\bm{w}}
\renewcommand{\u}{\bm{u}}
\newcommand{\0}{\bm{0}}

\newcommand{\z}{\bm{z}}

\newcommand{\poly}{\text{poly}}
\newcommand{\re}{\mathbb{R}}
\renewcommand{\v}{\bm{v}}
\newcommand{\as}{}
\newcommand{\bs}{}

\title{Differentially Private Online Learning}

\author{
Prateek Jain\\
Microsoft Research India\\
\href{mailto:prajain@microsoft.com}{prajain@microsoft.com}
\and
Pravesh Kothari\\
University of Texas at Austin\\
\href{mailto:kothari@cs.utexas.edu}{kothari@cs.utexas.edu}
\and
Abhradeep Thakurta\thanks{Part of the work was done while visiting Microsoft Research India.}\\
Pennsylvania State University\\
\href{mailto:azg161@cse.psu.edu}{azg161@cse.psu.edu}
}

\begin{document}

\date{}
\maketitle

\begin{abstract}
In this paper, we consider the problem of preserving privacy in the online learning setting. Online learning involves learning from the data in real-time, so that the learned model as well as its outputs are also continuously changing. 
 This makes preserving privacy of each data point significantly more challenging as its effect on the learned model can be easily tracked by changes in the subsequent outputs. Furthermore, with more and more online systems (e.g. search engines like Bing, Google etc.) trying to learn their customer's behavior by leveraging their access to sensitive customer data (through cookies etc), the problem of privacy preserving online learning has become critical as well.

We study the problem in the online convex programming (\ocp) framework---a popular online learning setting with several interesting theoretical and practical implications---while using differential privacy as the formal privacy measure. For this problem, we distill two critical attributes that a private OCP algorithm should have in order to provide reasonable privacy as well as utility guarantees: 1) linearly decreasing sensitivity, i.e., as new data points arrive their effect on the learning model decreases, 2) sub-linear regret bound---regret bound is a popular goodness/utility measure of an online learning algorithm. 
 Given an \ocp\ algorithm that satisfies these two conditions, we provide a general framework to convert the given algorithm into a privacy preserving \ocp\ algorithm with good (sub-linear) regret. We then illustrate our approach by converting two popular online learning algorithms into their differentially private variants while guaranteeing sub-linear regret ($O(\sqrt{T})$). Next, we consider the special case of online linear regression problems, a practically important class of online learning problems, for which we generalize an approach by \cite{DNPR} to provide a differentially private algorithm with just $O(\log^{1.5} T)$ regret. Finally, we show that our online learning framework can be used to provide differentially private algorithms for offline learning as well. For the offline learning problem, our approach obtains better error bounds as well as can handle larger class of problems than the existing state-of-the-art methods \cite{CM}.
\end{abstract}
\section{Introduction}\noindent
As computational resources are increasing rapidly, modern websites and online systems are able to process large amounts of information gathered from their customers in real time. While typically these websites intend to learn and improve their systems in real-time using the available data, this also represents a severe threat to the privacy of customers.

For example, consider a generic scenario for a web search engine like Bing. Sponsored advertisements (ads) served with search results form a major source of revenue for Bing, for which, Bing needs to serve ads that are relevant to the user and the query. As each user is different and can have different definition of ``relevance'', many websites typically try to learn the user behavior using past searches as well as other available demographic information. This learning problem has two key features: a) the advertisements are generated {\em online} in response to a query, b) feedback for goodness of an ad for a user cannot be obtained until the ad is served. Hence, the problem is an online learning game where the search engine tries to guess (from history and other available information) if a user would like an ad and gets the cost/reward only after making that online decision; after receiving the feedback the search engine can again update its model. This problem can be cast as a standard online learning problem and several existing algorithms can be used to solve it reasonably well. 

However, processing critical user information in real-time also poses severe threats to a user's privacy. For example, suppose Bing in response to certain past queries (let say about a disease), promotes a particular ad which otherwise doesn't appear at the top and the user clicks that ad. Then, the corresponding advertiser should be able to guess user's past queries, thus compromising privacy. Hence, it is critical for the search engine to use an algorithm which not only provides correct guess about relevance of an ad to a user, but also guarantees privacy to the user. Some of the other examples where privacy preserving online learning is critical are \emph{online portfolio management} \cite{KV}, \emph{online linear prediction} \cite{HKKA} etc.

In this paper, we address privacy concerns for online learning scenarios similar to the ones mentioned above.  Specifically, we provide a generic framework for privacy preserving online learning. We use \emph{differential privacy} \cite{DMNS} as the formal privacy notion, and use {\em online convex programming} (\ocp) \cite{Zin} as the formal online learning model.

Differential privacy is a popular privacy notion with several interesting theoretical properties. Recently, there has been a lot of progress in differential privacy. However, most of the results assume that all of the data is available beforehand and an algorithm processes this data to extract interesting information without compromising privacy. In contrast, in the online setting that we consider in this paper, data arrives online\footnote{At each time step one data entry arrives.} (e.g. user queries and clicks) and the algorithm has to provide an output (e.g. relevant ads) at each step. Hence, the number of outputs produced is roughly same as the size of the entire dataset. Now, to guarantee differential privacy one has to analyze privacy of the complete sequence of outputs produced, thereby making privacy preservation a significantly harder problem in this setting. In a related work, \cite{DNPR} also considered the problem of differential private online learning. Using the online experts model as the underlying online learning model, \cite{DNPR} provided an accurate differentially private algorithm to handle counting type problems. However, the setting and the class of problems handled by \cite{DNPR} is restrictive and it is not clear how their techniques can be extended to handle typical online learning scenarios, such as the one mentioned above. See Section~\ref{sec:related} for a more detailed comparison to \cite{DNPR}. 

Online convex programming (\ocp), that we use as our underlying online learning model, is an important and powerful online learning model with several theoretical and practical applications.
\ocp\ requires that the algorithm selects an output at each step from a fixed {\em convex} set, for which the algorithm incurs cost according to a {\em convex} function (that maybe different at each step). The cost function is revealed only after the point is selected. Now the goal is to minimize the {\em regret}, i.e., total ``added'' loss incurred in comparison to the optimal offline solution---a solution obtained after seeing all the cost functions. \ocp\ encompasses various online learning paradigms and has several applications such as portfolio management \cite{EC}. 
Now, assuming that each of the cost function is bounded over the fixed convex set, regret incurred by any \ocp\ algorithm can be trivially bounded by $O(T)$ where $T$ is the total number of time-steps for which the algorithm is executed. However, recently several interesting algorithms have been developed that can obtain regret that is sub-linear in $T$. That is, as $T\rightarrow \infty$, the total cost incurred is same as the cost incurred by the optimal offline solution. In this paper, we use regret as a ``goodness'' or ``utility'' property of an algorithm and require that a reasonable \ocp\ algorithm should at least have sub-linear regret.

To recall, we consider the problem of differentially private \ocp\,, where we want to provide differential privacy guarantees along with sub-linear regret bound. To this end, we provide a general framework to convert any online learning algorithm into a differentially private algorithm with sub-linear regret, provided that the algorithm satisfies two criteria: a) linearly decreasing sensitivity (see Definition~\ref{defn:sensitivity}), b) sub-linear regret. We then analyze two popular \ocp\, algorithms namely, Implicit Gradient Descent (\igd\,) \cite{KB10} and Generalized Infinitesimal Gradient Ascent (\giga\,) \cite{Zin} to guarantee differential privacy as well as $\tilde O(\sqrt{T})$ regret for a fairly general class of strongly convex, Lipschitz continuous gradient functions. In fact, we show that \igd\, can be used with our framework for even non-differentiable functions.
 We then show that if the cost functions are quadratic functions (e.g. online linear regression), then we can use another \ocp\ algorithm called Follow The Leader (FTL) \cite{HKKA,SK08} along with a generalization of a technique by \cite{DNPR} to guarantee $O(\ln ^{1.5} T)$ regret while preserving privacy.

Furthermore, our differentially private online learning framework can be used to obtain privacy preserving algorithms for a large class of offline learning problems \cite{CM} as well. 
In particular, we show that our private \ocp\ framework can be used to obtain good generalization error bounds for various offline learning problems using techniques from \cite{SKT} (see Section~\ref{sec:offlineLearn}). Our differentially private offline learning framework can handle a larger class of learning problems with better error bounds than the existing state-of-the-art methods \cite{CM}.
\subsection{Related Work}\noindent
\label{sec:related}
As more and more of world's information is being digitized, privacy has become a critical issue. To this end, several ad-hoc privacy notions have been proposed, however, most of them stand broken now. 
 De-anonymization of the Netflix challenge dataset by \cite{NV08} and of the publicly released AOL search logs \cite{NYT} are two examples that were instrumental in discarding these ad-hoc privacy notions. Even relatively sophisticated notions such as $k$-anonymity \cite{Sweeney} and $\ell$-diversity \cite{KG} have been permeated through by attacks \cite{GKS08}. 
 Hence, in pursuit of a theoretically sound notion of privacy , \cite{DMNS} proposed \emph{differential privacy}, a cryptography inspired definition of privacy. This notion has now been accepted as the standard privacy notion, and in this work we adhere to this notion for our privacy guarantees.

Over the years, the privacy community have developed differentially private algorithms for several interesting problems \cite{Dwork06,Dwork09,Dwork10}. 
In particular, there exists many results concerning privacy for learning problems \cite{BLR08,CM, WM10, PRR10,RBHT}. Among these, \cite{CM} is of particular interest as they consider a large class of learning problems that can be written as (offline) convex programs. Interestingly, our techniques can be used to handle the offline setting of \cite{CM} as well and in fact, our method can handle larger class of learning problems with better error bounds (see Section~\ref{sec:offlineLearn}).

As mentioned earlier, most of the existing work in differentially private learning has been in the offline setting where the complete dataset is provided upfront. One notable exception is the work of \cite{DNPR}, where authors formally defined the notion of differentially private learning when the data arrives online. Specifically, \cite{DNPR} defined two notions of differential privacy,  namely \emph{user level privacy} and \emph{event level privacy}. Roughly speaking, user level privacy guarantees are at the granularity of each user whose data is present in the dataset. In contrast, event level privacy provides guarantees at the granularity of individual records in the dataset. It has been shown in \cite{DNPR} that it is \emph{impossible} to obtain any non-trivial result with respect to user level privacy. In our current work we use the notion of event level privacy. \cite{DNPR} also looked at a particular online learning setting called the \emph{experts setting}, where their algorithm achieves a regret bound of $O(\ln^{1.5} T)$ for counting problems  while guaranteeing event level differential privacy.  However, their approach is restricted to experts advice setting, and cannot handle typical online learning problems that arise in practice. In contrast, we consider a significantly more practical and powerful class of online learning problems, namely, online convex programming, and also provide a method for handling a large class of offline learning problems. 

In a related line of work, there have been a few results that use online learning techniques to obtain differentially private algorithms \cite{HR, DRV}. In particular, \cite{HR} used experts framework to obtain a differentially private algorithm for answering \emph{adaptive} counting queries on a dataset. However, we stress that although these methods use online learning techniques, however they are designed to handle the offline setting only where the dataset is fixed and known in advance.

Recall that in the online setting, 
whenever a new data entry is added to $D$, a query has to be answered, i.e., the total number of queries to be answer is of the order of size of the dataset. In a line of work started by \cite{DiNi} and subsequently explored in details by \cite{DMT,KRSU}, it was shown that if one answers $O(T)$ subset sum queries on a dataset  $D\in\{0,1\}^T$ with noise in each query smaller than $\sqrt{T}$, then using those answers \emph{alone} one can reconstruct a large fraction of $D$. That is, when the number of queries is almost same as the size of dataset, then a reasonably ``large'' amount of noise needs to be added for preserving privacy. 
 Subsequently, there has been a lot of work in providing lower bounds (specific to differential privacy) on the amount of noise needed to guarantee privacy while answering a given number of queries (see \cite{HT10,KRSU,De}). We note that our generic online learning framework (see Section~\ref{sec:OCPPriv}) also adds noise of the order of $T^{0.5+c}$, $c>0$ at each step, thus respecting the established lower bounds. In contrast, our algorithm for quadratic loss function (see Section~\ref{sec:Log}) avoids this barrier by exploiting the special structure of queries that need to be answered. 
\subsection{Our Contributions}\noindent
\label{sec:contrib}
Following are the main contributions of this paper:
\begin{enumerate}
\item We formalize the problem of privacy preserving online learning using differential privacy as the privacy notion and Online Convex Programming (\ocp) as the underlying online learning model. We provide a generic differentially private framework for \ocp\ in Section \ref{sec:OCP} and provide privacy and utility (regret) guarantees.
\item We then show that using our generic framework, two popular \ocp\ algorithms, namely Implicit Gradient Descent (IGD) \cite{KB10} and Generalized Infinitesimal Gradient Ascent (GIGA) \cite{Zin} can be easily transformed into private online learning algorithms with good regret bound.
\item For a special class of \ocp\ where cost functions are quadratic functions only, we show that we can improve the regret bound to $O( \ln^{1.5} T)$ by exploiting techniques from \cite{DNPR}. This special class includes a very important online learning problem, namely, online linear regression.
\item In Section \ref{sec:offlineLearn} we show that our differentially private framework for online learning can be used to solve a large class of offline learning problems as well (where the complete dataset is available at once) and provide tighter utility guarantees than the existing state-of-the-art results \cite{CM}.
\item Finally, through empirical experiments on benchmark datasets, we demonstrate practicality of our algorithms for practically important problems of online linear regression, as well as, online logistic regression
 (see Section \ref{sec:EmpRel}).
\end{enumerate}
\section{Preliminaries}
\label{sec:prelim}
\subsection{Online Convex Programming}\noindent
\label{sec:prelim_ocp}
Online convex programming (\ocp\,) is one of the most popular and powerful paradigm in the online learning setting. \ocp\, can be thought of as a game between a player and an adversary. At each step $t$, player selects a point $\x_t\in \re^d$ from a convex set $\C$. Then, adversary selects a convex cost function $f_t: \re^d \rightarrow \re$ and the player has to pay a cost of $f_t(\x_t)$. Hence, an \ocp\ algorithm $\A$ maps a function sequence $F=\langle f_1, f_2,\dots,f_T \rangle$ to a sequence of points $X=\langle\x_1, \x_2, \dots, \x_{T}\rangle \in \C^T$, \ie $\A(F)=X$. Now, the goal of the player (or the algorithm) is to minimize the total cost incurred over a fixed number (say $T$) of iterations.  However, as adversary selects function $f_t$ after observing player's move $\x_t$, it can make the total cost incurred by the player arbitrarily large. Hence, a more realistic goal for the player is to minimize {\em regret}, i.e., the total cost incurred when compared to the optimal offline solution $\x^*$ selected in hindsight, i.e., when all the functions have already been provided. Formally,
\begin{defn}[Regret]
Let $\A$ be an online convex programming algorithm. Also, let $\A$ selects a point $\x_t\in \C$ at $t$-th iteration and $f_t:\re^d\rightarrow \re$ be a convex cost function served at $t$-th iteration. Then, the regret $\R_\A$ of $\A$ over $T$ iterations is given by:
$$\as \R_\A(T)=\sum_{t=1}^T f_t(\x_t)-\min_{\x^*\in \C}\sum_{t=1}^T f_t(\x^*).\bs$$
\label{defn:regret}
\end{defn}
Assuming $f_t$ to be a bounded function over $\C$, any trivial algorithm $\A$ that selects a random point $\x_t\in \C$ will have $O(T)$ regret. However, several results \cite{KB10, Zin} show that if each $f_t$ is a bounded Lipschitz function over $\C$, $O(\sqrt{T})$ regret can be achieved. Furthermore, if each $f_t$ is a ``strongly'' convex function, $O(\ln T)$ regret can be achieved \cite{KB10, SK08}.
\subsection{Differential Privacy}\noindent
\label{sec:prelim_dp}
We now formally define the notion of differential privacy in the context of our problem.
\begin{defn}[$(\epsilon,\delta)$-differential privacy~\cite{DMNS,ODO}]
\label{def:dp}
Let $F=\langle f_1, f_2, \dots, f_T\rangle$ be a sequence of convex functions. Let $\A(F)=X$, where  $X=\langle\x_1, \x_2, \dots, \x_{T}\rangle \in \C^T$ be $T$ outputs of \ocp\ algorithm $\A$ when applied to $F$. Then, a randomized \ocp\ algorithm $\mathcal{A}$ is $(\epsilon,\delta)$-differentially private if given any two function sequences $F$ and $F'$ that differ in at most one function entry, for all $\mathcal{S}\subset \C^{T}$ the following holds:
$$\Pr[\mathcal{A}(F)\in \mathcal{S}]\leq e^\epsilon\Pr[\mathcal{A}(F')\in \mathcal{S}]+\delta$$
\end{defn}
Intuitively, the above definition means that changing an $f_\tau\in F, \tau\leq T$ to some other function $f'_\tau$ will not modify the output sequence $X$ by a large amount. If we consider each $f_{\tau}$ to be some information associated with an individual, then the above definition states that the presence or absence of that individual's entry in the dataset will not affect the output by too much. Hence, output of the algorithm $\A$ will not reveal any extra information about the individual. Privacy parameters $(\epsilon,\delta)$ decides the extent to which an individual's entry affects the output; lower values of $\epsilon$ and $\delta$ means higher level of privacy. Typically, $\delta$ should be exponentially small in the problem parameters, i.e., in our case $\delta \approx \exp(-T)$. 
\subsection{Notation}\noindent
\label{sec:prelim_notation}
$F=\langle f_1, f_2, \dots, f_T\rangle$ denotes the function sequence given to an \ocp\ algorithm $\A$ and $\A(F)=X$ s.t. $X=\langle\x_1, \x_2, \dots, \x_{T}\rangle \in \C^T$ represents output sequence when $\A$ is applied to $F$. We denote the subsequence of functions $F$ till the $t$-th step as $F_t=\langle f_1, \dots, f_t\rangle$. $d$ denotes the dimensionality of the ambient space of  convex set $\C$. Vectors are denoted by bold-face symbols, matrices are represented by capital letters. $\bm{x}^T\bm{y}$ denotes the inner product of  $\x$ and $\bm{y}$. $\|M\|_2$ denotes spectral norm of matrix $M$; recall that for symmetric matrices $M$, $\|M\|_2$ is the largest eigenvalue of $M$.

Typically, $\alpha$ is  the minimum strong convexity parameter of any $f_t\in F$. Similarly, $L$ and $L_G$ are the largest Lipschitz constant and the Lipschitz constant of the gradient of any $f_t\in F$. Recall that a function $f:\C\rightarrow\mathbb{R}$ is $\alpha$-strongly convex, if for all $\gamma\in (0,1)$ and for all $\x,\bm y\in\C$ the following holds: $f(\gamma\x+(1-\gamma)\bm y)\leq \gamma f(\x)+(1-\gamma) f(\bm y)-\frac{\alpha}{2}\ltwo{\x-\bm y}^2$. Also recall that a function $f$ is $L$-Lipschitz, if for all $\x,\bm y\in\C$ the following holds: $|f(\x)-f(\bm y)|\leq L\ltwo{\x-\bm y}$. 
Function $f$ is Lipschitz continuous gradient if $\ltwo{\grad f(\x)-\grad f(\bm y)}\leq L_G\ltwo{\x-\bm y}$, for all $\x,\bm{y}\in\C$. Non-private and private versions of an \ocp\ algorithm outputs $\x_{t+1}$ and $\hat{\x}_{t+1}$ respectively, at time step $t$. $\x^*$ denotes the optimal offline solution, that is $\x^* = \argmin_{\x \in \C} \sum_{t = 1}^T f_t(\x)$.  $\R_{\A}(T)$ denotes  regret of an \ocp\  algorithm $\A$ when applied for $T$ steps.



\section{Differentially Private Online Convex Programming}\noindent
\label{sec:OCP}
In Section \ref{sec:prelim_ocp}, we defined the online convex programming (\ocp\ ) problem and presented a notion of utility (called \emph{regret}) for \ocp\ algorithms. Recall that a reasonable \ocp\ should have \emph{sub-linear} regret, \ie the regret should be sub-linear in the number of time steps $T$.

In this section, we present a generic differentially private framework for solving \ocp\ problems (see Algorithm \ref{alg:ocp_diff}). We further provide formal privacy and utility guarantees for this framework (see Theorems \ref{Thm:GenPri} and \ref{thm:util}). 
  We then use our private \ocp\ framework to convert two existing \ocp\ algorithms, namely, \emph{Implicit Gradient Decent} (\igd)\cite{KB10} and \emph{Generalized Infinitesimal Gradient Ascent }(\giga)\cite{Zin}  into differentially private algorithms using a ``generic'' transformation. For both the algorithms mentioned above, we guarantee $(3\epsilon,2\delta)$-differential privacy with sub-linear regret.


Recall that a differentially private \ocp\ algorithm should not produce a {\it significantly} different output for a function sequence $F_t'$ (with high probability) when compared to $F_t$, where $F_t$ and $F'_t$ differ in exactly one function. Hence, to show differential privacy for an \ocp\ algorithm, we first need to show that it is not very ``sensitive'' to previous cost functions. To this end, below we formally define \emph{sensitivity} of an \ocp\ algorithm $\A$. 
\begin{defn}[$L_2$-sensitivity~\cite{DMNS,CM}]
Let $F,F'$ be two function sequences differing in at most one entry, i.e., at most one function can be different. Then, the sensitivity of an algorithm~${\mathcal A}:F\rightarrow\C^T$ is the difference in the $t$-th output $\x_{t+1}=\A(F)_t$ of the algorithm $\A$, i.e.,
$$\as {\mathcal S}({\mathcal A},t)=\sup_{F,F'}\ltwo{\mathcal{A}(F)_t-\mathcal{A}(F')_t}\bs. $$
\label{defn:sensitivity}
\end{defn}

As mentioned earlier, another natural requirement for an \ocp\ algorithm is that it should have a provably low regret bound. There exists a variety of methods in literature which satisfy this requirement up to different degrees depending on the class of the functions $f_t$.

Under the above two assumptions on the \ocp\ algorithm $\A$, we provide a general framework for adapting the given \ocp\ algorithm ($\A$) into a differentially private algorithm. Formally, the given \ocp\ algorithm $\A$ should satisfy the following two conditions:
\begin{itemize}
\setlength{\partopsep}{-5pt}
\setlength{\topsep}{-5pt}
\setlength{\itemsep}{-5pt}
\item $L_2$-sensitivity: The $L_2$-sensitivity ${\mathcal S}(\A,t)$ of the algorithm $\A$ should decay linearly with time, \ie
\begin{equation}
\as
{\mathcal S}(\A,t)\leq \frac{\lambda_\A}{t},
\label{eq:sens_ocp}
\bs
\end{equation}
where $\lambda_\A>0$ is a constant depending only on $\A$, and strong convexity, Lipschitz constant of the functions in $F$.
\item Regret bound $\R_{\A}(T)$: Regret of $\A$ is assumed to be bounded, typically by a sub-linear function of $T$, \ie
\begin{equation}
\as
\R_\A(T)=\sum_{t=1}^T f_t(\x_t)-\min_{\x^*\in \C}\sum_{t=1}^T f_t(\x^*)= o(T).
\label{eq:reg_ocp}
\bs
\end{equation}
\end{itemize}
Given $\A$ that satisfies both \eqref{eq:sens_ocp} and \eqref{eq:reg_ocp}, we convert it into a private algorithm by perturbing $\x_{t+1}$ (output of $\A$ at $t$-th step) by a small amount of noise, whose magnitude is dependent on the sensitivity parameter $\lambda_\A$ of $\A$. Let $\tx_{t+1}$ be the perturbed output, which might be outside the convex set $\C$. As our online learning game requires each output to lie in $\C$, we project $\tx_{t+1}$ back to $\C$ and output the projection $\hat\x_{t+1}$. Note that, our Private \ocp\ (\pocp) algorithm also stores the ``uncorrupted'' iterate $\x_{t+1}$, as it would be used in the next step. See Algorithm~\ref{alg:ocp_diff} for a pseudo-code of our method.

Now, using the above two assumptions along with concentration bounds for Gaussian noise vectors, we obtain both privacy and regret bound for our Private \ocp\ algorithm. See Section \ref{sec:OCPPriv} and \ref{sec:OCPReg} for a detailed analysis of our privacy guarantee and the regret bound.

In Sections \ref{sec:IGD} and \ref{sec:GIGA}, we use our \emph{abstract} private \ocp\ framework to convert \igd\, and  \giga\, algorithms into private \ocp\ methods. 
For both the algorithms, privacy and regret guarantees follow easily from the guarantees of our \ocp\ framework once the corresponding sensitivity bounds are established.
\begin{algorithm}[tb]
	\caption{Private \ocp\ Method (\pocp)}
	\begin{algorithmic}[1]
		\STATE {\bfseries Input:} \ocp\ algorithm $\A$, cost function sequence $F=\langle f_1,\cdots,f_T\rangle$ and the convex set $\C$
                \STATE {\bfseries Parameter:} privacy parameters $(\epsilon, \delta)$
		\STATE Choose $\x_1$ and $\hat\x_1$ randomly from $\C$
		\FOR{$t=1$ to $T-1$}
                \STATE \textbf{Cost:} $\ \ L_t(\hat\x_t)=f_t(\hat\x_t)$
                \STATE \textbf{\ocp\ Update:} $\ \ \x_{t+1}\leftarrow \A(\langle f_1, \dots, f_t \rangle, \langle\x_1, \dots, \x_t \rangle, \C)$
                \STATE \textbf{Noise Addition:} $\ \ \tx_{t+1}\leftarrow \x_{t+1}+\b_{t+1}$,\ \ $\b_{t+1}\sim\mathcal{N}(0^d,\frac{\beta^2}{t^2}\I^d )$, where  \ $\beta= \lambda_\A T^{0.5+c} \sqrt{\frac{2}{\epsilon}\left(\ln\frac{T}{\delta}+\frac{\sqrt{\epsilon}}{T^{0.5+c}}\right)}$  and $c=\frac{\ln\frac{1}{2}\ln(2/\delta)}{2\ln T}$
                \STATE Output $\hat\x_{t+1}=\argmin_{\x\in \C}\left(\|\x-\tx_{t+1}\|_2^2\right)$
		\ENDFOR					
	\end{algorithmic}
	\label{alg:ocp_diff}
\end{algorithm}
\subsection{Privacy Analysis for \pocp}\noindent
\label{sec:OCPPriv}
Under the assumption \eqref{eq:sens_ocp}, changing one function in the cost function sequence $F$ can lead to a change of at most $\lambda_\A/t$ in the $t$-th output of $\A$. 
Hence, intuitively, adding a noise of the same order should make the $t$-th step output of Algorithm \ref{alg:ocp_diff} differentially private. We make the claim precise in the following lemma.
\begin{lemma}
Let $\A$ be an \ocp\ algorithm that satisfies sensitivity assumption \eqref{eq:sens_ocp}. Also, let $c>0$ be any constant and $\beta=\lambda_\A T^{0.5+c}\sqrt{\frac{2}{\epsilon}\left(\ln\frac{T}{\delta}+\frac{\sqrt{\epsilon}}{T^{0.5+c}}\right)}$. Then, the $t$-th step output of Algorithm~\ref{alg:ocp_diff}, $\hatx_{t+1}$, is $(\frac{\sqrt\epsilon}{T^{0.5+c}},\frac{\delta}{T})$-differentially private.
\label{Lem:SingGenPriv}
\end{lemma}
\begin{proof}
As the output $\hatx_{t+1}$ is just a projection, i.e., a function (independent of the input functions $F$) of $\tx_{t+1}$, hence $(\epsilon, \delta)$-differential privacy for $\tx_{t+1}$ would imply the same for $\hatx_{t+1}$.

Now by the definition of differential privacy (see Definition~\ref{def:dp}), $\txt$ is $(\epsilon_1, \frac{\delta}{T})$-differential private, if for any measurable set $\Omega\subseteq\mathbb{R}^{p}$:
$$\Pr[\txt\in \Omega]\leq e^{\epsilon_1}\Pr[\txt'\in \Omega]+\delta/T,$$
where $\txt=\xt+b$ is the output of the noise addition step (see Algorithm~\ref{alg:ocp_diff}, Step 7) of our \pocp\ algorithm, when applied to function sequence $F_t=\langle f_1, \dots, f_t\rangle$. Similarly, $\txt'=\xt'+b$ is the output of the noise addition to $\xt'$ which is obtained by applying update step to $F_t'$, where $F_t'$ differs from $F_t$ in exactly one function entry.

Now, $\txt\sim {\mathcal N}(\x_{t+1},\frac{\beta^2}{t^2} \mathbb{I}^d)$ and $\txt'\sim {\mathcal N}(\x_{t+1}', \frac{\beta^2}{t^2} \mathbb{I}^d)$. Let $\dxt=\xt-\xt'$. Then, we have $(\txt-\xt)^T\dxt\sim\mathcal{N}(0,\frac{\beta^2}{t^2}\|\dxt\|_2^2)$. 
Now, using assumption \eqref{eq:sens_ocp} for the \ocp algorithm ${\mathcal A}$ and Mill's inequality, $$\Pr\left[\left|(\txt-\xt)^T\dxt\right|\geq \frac{\beta\lambda_\A}{t^2}z\right]\leq \Pr\left[\left|(\txt-\xt)^T\dxt\right|\geq\frac{\beta}{t}\|\xt-\xt'\|z\right]\leq e^{-\frac{z^2}{2}},$$ where $z>0$. Setting R.H.S. $\leq\frac{\delta}{T}$, we have $z\geq\sqrt{2\ln\frac{T}{\delta}}$.

Now, we define a ``good set'' $\G$:
\begin{equation}
\label{eq:good}
\x\in \G \text{  iff  } \left|(\x-\xt)^T\dxt\right|\geq\frac{\beta\lambda_\A}{t^2}z.
\end{equation}
Note that, \begin{equation}\label{eq:gg}\Pr[\txt \not\in \G]=\Pr\left[\left|(\txt-\xt)^T\dxt\right|\geq \frac{\beta\lambda_\A}{t^2}z\right]\leq \frac{\delta}{T}.\end{equation}
We now bound $\Pr[\txt\in \Omega]$:
\begin{equation}
  \label{eq:txtup}
  \Pr[\txt\in \Omega] \leq \Pr[\txt\in \Omega\cap \G] + \Pr[\txt \not\in \G]\leq  \Pr[\txt\in \Omega\cap \G] + \frac{\delta}{T}.
\end{equation}
As $\txt\sim {\mathcal N}(\x_{t+1},\frac{\beta^2}{t^2} \mathbb{I}^d)$,
\begin{equation}
  \label{eq:txtog}
  \Pr[\txt\in \Omega\cap \G]=\int_{\x\in \Omega\cap \G}\exp\left(-\frac{\ltwo{\x-\xt}^2}{2\frac{\beta^2}{t^2}}\right)d\x.
\end{equation}
Now, for $\x\in \Omega\cap \G$:
\begin{align}
  \frac{\exp\left(-\frac{t^2 \ltwo{\x-\xt}^2}{2\beta^2}\right)}{\exp\left(-\frac{t^2\ltwo{\x-\xt'}^2}{2\beta^2}\right)}&=\exp\left(\frac{t^2}{2\beta^2}\dxt^T(2\x-\xt-\xt')\right),\nonumber\\
&=\exp\left(\frac{t^2}{2 \beta^2}\left(2\dxt^T(\x-\xt)-\|\dxt\|_2^2\right)\right), \nonumber\\
&\leq \exp\left(\frac{t^2}{2 \beta^2 }\left(2|\dxt^T(\x-\xt)|+\|\dxt\|_2^2\right)\right), \nonumber\\
&\leq \exp\left(\frac{\lambda_\A}{\beta}\sqrt{2\ln\frac{T}{\delta}}+\frac{\lambda_\A^2}{2\beta^2}\right),\nonumber\\
&\leq e^{\epsilon_1},
\label{eq:HPB}
\end{align}
where $\epsilon_1=\frac{\sqrt\epsilon}{T^{0.5+c}}$ and $\beta$ is as given in the Lemma statement. The second last inequality follows from the definition of $\G$ and the sensitivity assumption~\eqref{eq:sens_ocp}.

Hence, using \eqref{eq:txtup}, \eqref{eq:txtog}, and \eqref{eq:HPB}, we get:
\begin{align}
  \Pr[\txt\in \Omega] \leq \int_{\x\in \Omega\cap \G}  e^{\epsilon_1}\exp\left(-\frac{t^2\ltwo{\x-\xt'}^2}{2\beta^2}\right) + \frac{\delta}{T}\leq e^{\epsilon_1}\Pr[\txt'\in \Omega]+\frac{\delta}{T}.
  \label{eq:txtup1}
\end{align}
Hence, proved.
\end{proof}
Now, the above lemma shows $(\frac{\sqrt\epsilon}{T^{0.5+c}},\frac{\delta}{T})$-differential privacy for each step of Algorithm~\ref{alg:ocp_diff}. Hence, using a simple composition argument (see \cite{DL}) should guarantee $(T^{0.5-c}\sqrt\epsilon, \delta)$-differential privacy for all the steps. So to get overall $\epsilon$ privacy, we will need $c=0.5$. That is, a noise of the order $O(T/t)$ needs to be added at each step, which intuitively means that the noise added is larger than the effect of incoming function $f_t$ and hence can lead to an arbitrarily bad regret.

To avoid this problem, we need to exploit the interdependence between the iterates (and outputs) of our algorithm so as to obtain a better bound than the one obtained by using the union bound. For this purpose, we use the following lemma by \cite{DRV} that bounds the relative entropy of two random variables in terms of the $L_\infty$ norm of their probability density ratio and also a proof technique developed by \cite{HR,HLM} for the problem of releasing differentially private datasets.
\begin{lemma}[\cite{DRV}]
Suppose two random variables $Y$ and $Z$ satisfy, $$\as\bs D_\infty(Y||Z)=\max_{w\in\supp(Y)}\ln\left(\frac{\pdf[Y=w]}{\pdf[Z=w]}\right)\leq\epsilon,\qquad \qquad \ D_{\infty}(Z||Y)\leq\epsilon.$$ Then $D(Y||Z)=\int_{w\in\supp(Y)}\pdf[Y=w]\ln\left(\frac{\pdf[Y=w]}{\pdf[Z=w]}\right)\leq 2\epsilon^2$. $supp(Y)$ is the support set of a random variable $Y$.
\label{lem:Dwork}
\end{lemma}
We now state a technical lemma which will be useful for our differential privacy proof.
\begin{lemma}
Assuming that at each stage $t$,  Algorithm~\ref{alg:ocp_diff} preserves $\frac{\sqrt{\epsilon}}{T^{0.5+c}}$-differential privacy, $$\E_{ \txt}\left[\ln\left(\frac{\pdf[\txt]}{\pdf[\txt'=\txt]}\right)\right]\leq\frac{2\epsilon}{T^{1+2c}},$$
where $\txt$ and $\txt'$ are output of the $t$-th iteration of the \textbf{Noise Addition Step} of our \pocp\ algorithm (Algorithm~\ref{alg:ocp_diff}), when applied to function sequences $F_t$ and $F_t'$ differing in exactly one function entry.
\label{lem:pro}
\end{lemma}
\begin{proof}
Using the fact that $\txt$ is $\frac{\sqrt{\epsilon}}{T^{0.5+c}}$-differential private:
$$\forall \x,\hspace{0.1cm}-\frac{\sqrt{\epsilon}}{T^{0.5+c}}\leq\ln\left(\frac{\pdf[\txt=\x]}{\pdf[\txt'=\x]}\right)\leq \frac{\sqrt{\epsilon}}{T^{0.5+c}}.$$
Lemma now follows using the above observation with Lemma~\ref{lem:Dwork}.
\end{proof}
Now we state the privacy guarantee for Algorithm \ref{alg:ocp_diff} over {\em all} $T$ iterations.
\begin{thm}[\pocp\ Privacy]
Let $\A$ be an \ocp\ algorithm that satisfies the sensitivity assumption \eqref{eq:sens_ocp}, then the \pocp\ algorithm (see Algorithm~\ref{alg:ocp_diff}) is $(3\epsilon,2\delta)$-differentially private.
\label{Thm:GenPri}
\end{thm}
\begin{proof}
Following the notation from the proof of Lemma~\ref{Lem:SingGenPriv}, let $\G$ be defined by \eqref{eq:good}. Now, using \eqref{eq:gg}, for each round,
\begin{equation}\as \bs
\Pr[\txt\not\in\G]\leq\frac{\delta}{T}.
\end{equation}
Now, the probability that the noise vectors $\b_{t+1}=\txt-\xt=\txt'-\xt', 1\leq t\leq T-1$ are all from the ``good'' set $\G$ in all the $T$ rounds is at least  $1-T\cdot \frac{\delta}{T}=1-\delta$.

We now condition the remaining proof on the event that the noise vector $\b_{t+1}$ in each round is such that $\tx_{t+1}\in \G$.

Let $L(\tx_1,\cdots,\tx_T)=\sum_{t=1}^T\ln\left(\frac{\pdf[\tx_t]}{\pdf[\tx_t'=\tx_t]}\right)$. 
Using Lemma~\ref{lem:pro}, $$\as\bs \E_{\tx_1, \cdots, \tx_T}[L(\tx_1,\cdots,\tx_T)]=\sum_{t=1}^T \E_{\tx_t}\left[\ln\left(\frac{\pdf[\tx_t]}{\pdf[\tx_t'=\tx_t]}\right)\right]\leq \frac{2T\epsilon}{T^{1+2c}}\leq\frac{2\epsilon}{T^{2c}}\leq 2\epsilon.$$

Let $Z_t=\ln\left(\frac{\pdf[\tilde{\x}_t]}{\pdf[\tilde{\x}_t'=\tx_t]}\right)$. Since each $b_t$ is sampled independently and the randomness in $Z_t$ is  only due to $b_t$, $Z_t$'s are independent. We have $L(\tx_1,\cdots,\tx_T)=\sum_{t=1}^T Z_t$, where $|Z_t|\leq\frac{\sqrt{\epsilon}}{T^{0.5+c}}$. By Azuma-Hoeffding's inequality,
\begin{align*}\as \bs
\Pr[L(\tx_1,\cdots,\tx_T)\geq 2\epsilon+\epsilon]\leq 2\exp\left(\frac{-2\epsilon^2}{T\times\frac{\epsilon}{T^{1+2c}}}\right)\leq 2\exp\left(-2T^{2c}\right).
\setlength{\belowdisplayskip}{-10pt}
\end{align*}
Setting $\delta=2\exp\left(-2T^{2c}\right)$, we get 
 $c=\frac{(\ln (\frac{1}{2}\ln \frac{2}{\delta} )}{2\ln T}$. Hence, with probability at least $1-\delta$, $3\epsilon$-differential privacy holds conditioned on $\tx_t\in\G$, i.e,
$$\setlength{\abovedisplayskip}{1pt}\bs\forall\z_1, \dots, \z_T\in \mathbb{R}^d, \hspace{0.1cm}\Pi_{t=1}^T\pdf(\tx_t=\z_t)\leq e^{3\epsilon}\Pi_{t=1}^T\pdf(\tx_t'=\z_t).$$
Also, recall that with probability at least $1-\delta$, the noise vector $\b_{t}$ in each round itself was such that $\tx_t\in \G$. Hence, with probability at least $1-2\delta$, $3\epsilon$-differential privacy holds. $(3\epsilon, 2\delta)$-differential privacy now follows using a standard argument similar to \eqref{eq:txtup}. 
\end{proof}
\subsection{Utility (Regret) Analysis for \pocp}\noindent
\label{sec:OCPReg}
In this section, we provide a generic regret bound analysis for our \pocp\ algorithm (see Algorithm \ref{alg:ocp_diff}). The regret bound of \pocp\ depends on the regret $\R_\A(T)$ of the non-private \ocp\ algorithm $\A$. For typical \ocp\ algorithms like \igd, \giga\ and \ftl\ , $\R_\A(T)=O(\log T)$, assuming each cost function $f_t$ is strongly convex.
\begin{thm}[\pocp\ Regret]
Let $L>0$ be the maximum Lipschitz constant of any function $f_t$ in the sequence $F$, 
 $\R_{\A}(T)$, the regret of the non-private \ocp\ algorithm $\A$ over $T$-time steps and $\lambda_\A$, the sensitivity parameter of $\A$ (see \eqref{eq:sens_ocp}). Then the expected regret of our \pocp\ algorithm (Algorithm~\ref{alg:ocp_diff}) satisfies:
$$\as\bs \E\left[\sum_{t=1}^Tf_t(\hat\x_t)\right]-\min_{\x\in\C}\sum_{t=1}^Tf_t(\x)\leq 2\sqrt{d}L(\lambda_\A+\|\C\|_2)\sqrt{T}\frac{\ln^2\frac{T}{\delta}}{\sqrt{\epsilon}}+\R_{\A}(T),$$
where $d$ is the dimensionality of the output space, and $\|\C\|_2$ is the diameter of the convex set $\C$. In other words, the regret bound is $ R_\A(T)+\tilde{O}\left(\sqrt{dT}\right)$.
\label{thm:util}
\end{thm}
\begin{proof}
Let $\hat\x_1,\cdots,\hat  \x_T$ be the output of the POCP algorithm. By the Lipschitz continuity of the cost functions $f_t$ we have,
\begin{eqnarray}
\sum_{t = 1}^T f_t(\hat{\x}_t)-\min_{\x\in\C}\sum_{t=1}^Tf_t(\x) \leq \sum_{t = 1}^T f_t(\x_t)-\min_{\x\in\C}\sum_{t=1}^Tf_t(\x) + L \sum_{t = 1 }^T \ltwo{\hat{\x}_t - \x_t}\leq R_\A(T) +  L \sum_{t = 1 }^T \ltwo{\hat{\x}_t - \x_t}.
\label{eq:rahx}
\end{eqnarray}
Since at any time $t\geq 1$,  $\hat\x_{t}$ is the projection of $\tilde\x_{t}$ on the convex set $\C$, we have $$\ltwo{\x_{t+1}-\hat\x_{t+1}}\leq\ltwo{\x_{t+1}-\tilde{\x}_{t+1}}=\ltwo{\b_{t+1}},\ \ \ \forall 1\leq t\leq T-1,$$ where $\b_{t+1}$ is the noise vector added in the $t$-th iteration of the \pocp\ algorithm. Therefore,
\begin{equation}
L \sum_{t=1}^T \ltwo{\x_{t}-\hat\x_{t}}\leq L \left(\|\C\|_2+\sum_{t =1}^{T-1}\ltwo{\b_{t+1}}\right).
\label{eq:hxb}
\end{equation}
Now, $\b_{t+1} \sim {\mathcal N} ( \0^d, \frac{\beta^2}{t^2} \I^d) $ where $$\beta= \lambda_{\A}T^{0.5+c}\sqrt{\frac{2}{\epsilon}\left(\ln\frac{T}{\delta}+\frac{\sqrt{\epsilon}}{T^{0.5+c}}\right)}.$$

Therefore, $\ltwo{\b_{t+1}}$ follows Chi-distribution with parameters $\mu=\frac{\sqrt{2}\frac{\beta}{t} \Gamma((d+1)/2)}{\Gamma(d/2)}$ and $\sigma^2=\frac{\beta^2}{t^2}(d-\mu^2)$.

Using $c = \frac{\ln{ (\frac{1}{2}\ln{\frac{2}{\delta}})} }{2\ln{T}}$,
\begin{align}
 \E[\sum_{t=1}^{T-1}\ltwo{\b_{t+1}}] &\leq \frac{\sqrt{2}\beta \Gamma((d+1)/2)}{\Gamma(d/2)} \int_{1}^{T-1} \frac{1}{t}dt, \nonumber\\
&\leq \frac{\Gamma((d+1)/2)}{\Gamma(d/2)} \lambda_\A\sqrt{T}\ln {T}\sqrt{\frac{2}{\epsilon}\ln{\frac{2}{\delta}}\left(\ln\frac{T}{\delta}+\frac{\sqrt{\epsilon}}{\sqrt{\frac{T}{2}\ln{\frac{2}{\delta}}}}\right)},\nonumber\\
&\leq 2\sqrt{d}\lambda_\A\sqrt{T}\frac{\ln^2\frac{T}{\delta}}{\sqrt{\epsilon}}.
\label{eq:noisenorm}
\end{align}
The theorem now follows by combining \eqref{eq:rahx}, \eqref{eq:hxb}, \eqref{eq:noisenorm}.
\end{proof}
\noindent Using Chebyshev's inequality, we can also obtain a high probability bound on the regret. 
\begin{corollary}
Let $L>0$ be the maximum Lipschitz constant of any function $f_t$ in the sequence $F$, 
$\R_{\A}(T)$ , the regret of the non-private \ocp\ algorithm $\A$ over $T$-time steps and $\lambda_\A$, the sensitivity parameter of $\A$ (see \eqref{eq:sens_ocp}). Then with probability at least $1-\gamma$,the regret of our Private \ocp\ algorithm (Algorithm~\ref{alg:ocp_diff}) satisfies:
$$\as\bs\sum_{t=1}^Tf_t(\hat\x_t)-\min_{\x\in\C}\sum_{t=1}^Tf_t(\x)\leq 2\sqrt{d}L(\lambda_\A+\|\C\|_2)\sqrt{T}\frac{\ln^2\frac{T}{\delta}}{\sqrt{\epsilon\gamma}}+\R_{\A}(T),$$
where $d$ is the dimensionality of the output space, $\|\C\|_2$ is the diameter of $\C$.
\label{corr:util}
\end{corollary}
\subsection{Implicit Gradient Descent Algorithm}\noindent
\label{sec:IGD}
In this section, we consider the Implicit Gradient Descent (\igd) algorithm \cite{KB10}, a popular online convex programming algorithm, and present a differentially private version of the same using our generic framework (see Algorithm~\ref{alg:ocp_diff}). Before deriving its privacy preserving version, we first briefly describe the \igd\ algorithm \cite{KB10}.

At each step $t$, \igd\ incurs loss $f_t(\x_t)$. Now, given $f_t$, \igd\ finds the $t$-th step output $\x_{t+1}$ so that
 it not ``far'' away from the current solution $\x_t$ but at the same time tries to minimize the cost $f_t(\x_{t+1})$. Formally,
\begin{equation}\as\bs
\igd:\qquad  \x_{t+1}\leftarrow \argmin_{\x\in \C} \frac{1}{2}\ltwo{\x-\x_t}^2+\eta_t f_t(\x),
  \label{eqn:abc1}
\end{equation}
where squared Euclidean distance is used as the notion of distance from the current iterate. \cite{KB10} describe a much large class of distance functions that can be used, but for simplicity of exposition we consider the Euclidean distance only. Assuming each $f_t(x)$ is a strongly convex function, a simple modification of the proof by \cite{KB10} shows $O(\log T)$ regret for \igd, i.e. $\R_{\igd}(T)=O(\log T)$.

\begin{algorithm}[tb]
	\caption{Private Implicit Gradient Descent (\pigd)}
	\begin{algorithmic}[1]
		\STATE {\bfseries Input:} Cost function sequence $F=\langle f_1,\cdots,f_T\rangle$ and the convex set $\C$
                \STATE {\bfseries Parameter:} privacy parameters $(\epsilon, \delta)$, maximum Lipschitz constant $L$ and minimum strong convexity parameter $\alpha$ of any function in $F$
		\STATE Choose $\x_1$ and $\hat\x_1$ randomly from $\C$
		\FOR{$t=1$ to $T-1$}
                \STATE \textbf{Cost:} $L_t(\hat\x_t)=f_t(\hat\x_t)$
                \STATE \textbf{Learning rate:} $\eta_t=\frac{1}{\alpha t}$
                \STATE \textbf{IGD Update:} $\x_{t+1}\leftarrow \argmin_{\x\in \C} \left(\frac{1}{2} \|\x-\x_t\|_2^2+\eta_t f_t(\x)\right)$
                \STATE \textbf{Noise Addition:} $\txt\leftarrow \xt+\b_{t+1},\ \ \b_{t+1}\sim\mathcal{N}(\0^d,\frac{\beta^2}{t^2}\I^d )$, where $\beta=2 L T^{0.5+c}\sqrt{\frac{2}{\epsilon}\left(\ln\frac{T}{\delta}+\frac{\sqrt{\epsilon}}{T^{0.5+c}}\right)}$ and $c=\frac{\ln\frac{1}{2}\ln(2/\delta)}{2\ln T}$
                \STATE Output $\hat\x_{t+1}=\argmin_{\x\in \C}\left(\|\x-\tx_{t+1}\|_2^2\right)$
		\ENDFOR					
	\end{algorithmic}
	\label{alg:igd}
\end{algorithm}
Recall that our generic private \ocp\ framework can be used to convert any \ocp\ algorithm as long as it satisfies low-sensitivity and low-regret assumptions (see \eqref{eq:sens_ocp}, \eqref{eq:reg_ocp}). Now, similar to \pocp\ , our Private \igd\ (\pigd) algorithm also adds an appropriately calibrated noise at each update step to obtain differentially private outputs $\hat{\x_{t+1}}$. See Algorithm~\ref{alg:igd} for a pseudo-code of our algorithm.

As stated above, $\R_{\igd}(T)=O(\log T)$ if each $f_t(x)$ is strongly convex. We now bound sensitivity of \igd\ at each step in the following lemma. The proof makes use of a simple and novel induction based technique.
\begin{lemma}[{\bf \igd\ Sensitivity}]
\label{lemma:sens_igd}
$L_2$-sensitivity (see Definition~\ref{defn:sensitivity}) of the \igd\ algorithm is $\frac{2L}{t}$ for the $t$-th iterate, where $L$ is the maximum Lipschitz constant of any function $f_{\tau}, 1\leq \tau\leq t$.
\end{lemma}
\begin{proof}
We prove the above lemma using mathematical induction. \\
{\bf Base Case ($t=1$)}: As $\x_1$ is selected randomly, it's value doesn't depend on the underlying dataset. \\
{\bf Induction Step $t=\tau+1$}: As $f_\tau$ is $\alpha$ strongly convex, the strong convexity coefficient of the function $\tilde{f}_\tau(\x)=\frac{1}{2}\|\x-\x_{\tau}\|_2^2+\eta_\tau f_\tau(\x)$ is $\frac{\tau+1}{\tau}$. Now using strong convexity and the fact that at optima $\x_{\tau+1}$, $\langle\grad \tilde{f}_\tau(\x_{\tau+1}), \x-\x_{\tau+1}\rangle\geq 0, \forall \x\in \C,$ we get:
\begin{equation}
  \label{eq:sens_opt_r}
  \tilde{f}_\tau(\x_{\tau+1}')\geq \tilde{f}_\tau(\x_{\tau+1})+\frac{\tau+1}{2\tau}\|\x_{\tau+1}-\x_{\tau+1}'\|_2^2.
\end{equation}
Now, we consider two cases:
\begin{itemize}
\item {\bf $F-F'=\{f_\tau\}$}: Define $\tilde{f}_\tau'(\x)=\frac{1}{2}\|\x-\x_{\tau}\|^2+\eta_\tau f_\tau'(\x)$ and let $\x_{\tau+1}'=\argmin_{\x\in \C} \tilde{f}_\tau'(\x)$. Then, similar to \eqref{eq:sens_opt_r}, we get:
\begin{equation}
  \label{eq:sens_opt_rp1}
  \tilde{f}_\tau'(\x_{\tau+1})\geq \tilde{f}_\tau'(\x_{\tau+1}')+\frac{\tau+1}{2\tau}\|\x_{\tau+1}-\x_{\tau+1}'\|_2^2.
\end{equation}
Adding \eqref{eq:sens_opt_r} and \eqref{eq:sens_opt_rp1}, we get:
$$\|\x_{\tau+1}-\x_{\tau+1}'\|_2^2\leq \frac{1}{\tau+1}|f_\tau(\x_{\tau+1}')+f_\tau'(\x_{\tau+1})-f_\tau(\x_{\tau+1})-f_\tau'(\x_{\tau+1}')|\leq \frac{2L}{\tau+1}\|\x_{\tau+1}-\x_{\tau+1}'\|_2.$$
Lemma now follows using simplification.
\item {\bf $F-F'=\{f_i\},\ i<\tau$}: Define $\tilde{f}_\tau'(\x)=\frac{1}{2}\|\x-\x_\tau'\|^2+\eta_\tau f_\tau(\x)$ and let $\x_{\tau+1}'=\argmin_{\x\in \C} \tilde{f}_\tau'(\x)$. Then, similar to \eqref{eq:sens_opt_r}, we get:
\begin{equation}
  \label{eq:sens_opt_rp11}
  \tilde{f}_\tau'(\x_{\tau+1})\geq \tilde{f}_\tau'(\x_{\tau+1}')+\frac{\tau+1}{2\tau}\|\x_{\tau+1}-\x_{\tau+1}'\|_2^2.
\end{equation}
Adding \eqref{eq:sens_opt_r} and \eqref{eq:sens_opt_rp11}, we get:
$$\|\x_{\tau+1}-\x_{\tau+1}'\|_2^2\leq \frac{\tau}{\tau+1}|(\x_{\tau+1}-\x_{\tau+1}')\cdot(\x_\tau-\x_\tau')|\leq \frac{\tau}{\tau+1}\|\x_{\tau+1}-\x_{\tau+1}'\|_2\|\x_{\tau}-\x_{\tau}'\|_2.$$
Lemma now follows after simplification and using the induction hypothesis.
\end{itemize}
\end{proof}
Using the above lemma and Theorem \ref{Thm:GenPri}, privacy guarantee for \pigd\ follows directly.
\begin{thm}[\pigd\ Privacy]
\pigd\ (see Algorithm~\ref{alg:igd}) is $(3\epsilon,2\delta)$-differentially private.
\label{Thm:GenPri_igd}
\end{thm}
Next, the utility (regret) analysis of our \pigd\ algorithm follows directly using Theorem \ref{thm:util} along with regret bound of \igd\ algorithm, $\R_{\igd}(T)=O(\frac{L^2}{\alpha}\log T+\ltwo{\C})$. 
Regret bound provided below scales roughly as $\tilde O(\sqrt{T})$.
\begin{thm}[\pigd\ Regret]
Let $L$ be the maximum Lipschitz constant and let $\alpha$ be the minimum strong convexity parameter of any function $f_t$ in the function sequence $F$. Then the expected regret of the private \igd\ algorithm over $T$-time steps is $\tilde O(\sqrt{T})$. Specifically,
$$\as\bs\E[\sum_{t=1}^Tf_t(\hat\x_t)]-\min_{\x\in\C}\sum_{t=1}^Tf_t(\x))\leq C\left(\frac{(L^2/\alpha+\|\C\|_2)\sqrt{d}\ln^{1.5} \frac{T}{\delta}}{\sqrt{\epsilon}}\sqrt{T}\right),$$
where $C>0$ is a constant and $d$ is the dimensionality of the output space.
\label{thm:util_igd}
\end{thm}



\subsection{Private GIGA Algorithm}\noindent
\label{sec:GIGA}
In this section, we apply our general differential privacy framework to the Generalized Infinitesimal Gradient Ascent (\giga) algorithm \cite{Zin}, which is one of the most popular algorithms for \ocp. \giga\ is a simple extension of the classical projected gradient method to the \ocp\ problem. Specifically, the iterates $\x_{t+1}$ are obtained by a projection onto the convex set $\C$, of the output of the gradient descent step $\x_t-\eta_t\grad f_t(\x_t)$ where $\eta_t=1/\alpha t$, and $\alpha$ is the minimum strong convexity parameter of any function $f_t$ in $F$.

For the rest of this section, we assume that each of the function $f_t$ in the input function sequence $F$ are {\em differentiable}, Lipschitz continuous gradient and strongly convex. Note that this is a stricter requirement than our private \igd\ algorithm where we require only the Lipschitz continuity of $f_t$.

Proceeding as in the previous section, we obtain a privacy preserving version of the \giga\ algorithm  using our generic \pocp\ framework (See Algorithm~\ref{alg:ocp_diff}). Algorithm~\ref{alg:giga_diff} details the steps involved in our Private \giga\ (\pgiga) algorithm. Note that \pgiga\ has an additional step (Step 3) compared to \pocp\, (Algorithm \ref{alg:ocp_diff}). This step is required to prove the sensitivity bound in Lemma \ref{lemma:sens_giga} given below.

\begin{algorithm}[tb]
	\caption{Private \giga\ (\pgiga)}
	\begin{algorithmic}[1]
		\STATE {\bfseries Input:} Cost function sequence $F=\langle f_1,\cdots,f_T\rangle$ and the convex set $\C$
                \STATE {\bfseries Parameter:} Privacy parameters $(\epsilon, \delta)$, Lipschitz continuity ($L$) and strong convexity ($\alpha$) bound on the function sequence $F$, $t_q=2L_G^2/\alpha^2$
		\STATE Choose $\x_1, \dots, \x_{t_q-1}$ and $\hat\x_1, \dots, \hat\x_{t_q-1}$ randomly from $\C$, incurring a cost of $\sum_{t=1}^{t_q-1} f_t(\hat\x_t)$
		\FOR{$t=t_q$ to $T-1$}
                \STATE \textbf{Cost:} $\ \ L_t(\hat\x_t)=f_t(\hat\x_t)$
                \STATE \textbf{Step Size:} $\eta_t=\frac{2}{\alpha t}$
                \STATE \textbf{\giga\ Update:} $\ \ \x_{t+1}\leftarrow \argmin_{x\in \C}\left(\|\x_t-\eta_t\grad f_t(\x_t)\|_2^2\right)$
                \STATE \textbf{Noise Addition:} $\ \ \tx_{t+1}\leftarrow \x_{t+1}+\b_{t+1}$,\ \ $\b_{t+1}\sim\mathcal{N}(\0^d,\frac{\beta^2}{t^2}\I^d )$, where  \ $\beta= 2G T^{0.5+c} \sqrt{\frac{2}{\epsilon}\left(\ln\frac{T}{\delta}+\frac{\sqrt{\epsilon}}{T^{0.5+c}}\right)}$ where $c=\frac{\ln\frac{1}{2}\ln(2/\delta)}{2\ln T}$
                \STATE Output $\hat\x_{t+1}=\argmin_{\x\in \C}\left(\|\x-\tx_{t+1}\|_2^2\right)$
		\ENDFOR					
	\end{algorithmic}
	\label{alg:giga_diff}
\end{algorithm}

Furthermore, we provide the privacy and regret guarantees for our \pgiga\ algorithm using Theorem~\ref{Thm:GenPri} and Theorem~\ref{thm:util}. To this end, we first show that \giga\ satisfies the sensitivity assumption mentioned in \eqref{eq:sens_ocp}.
\begin{lemma}[{\bf \giga\ Sensitivity}]
\label{lemma:sens_giga}
Let $\alpha>0$ be the minimum strong convexity parameter of any function $f_t$ in the function sequence $F$. Also, let $L_G$ be the maximum Lipschitz continuity parameter of the gradient of any function $f_t\in F$ and let $G=\max_\tau \|\grad f_t(x)\|_2, \forall x\in \C$. Then, $L_2$-sensitivity (see Definition~\ref{defn:sensitivity}) of the \giga\ algorithm is $\frac{2G}{\alpha t}$ for the $t$-th iterate, where $1\leq t\leq T$.
\end{lemma}
\begin{proof}
Let $\x_{t+1}$ and $\txt'$ be the $t$-th iterates when \giga\ is applied to $F$ and $F'$, respectively. Using this notation, to prove the $L_2$ sensitivity of \giga, we need to show that:
$$\|\x_{t+1}-\x_{t+1}'\|\leq \frac{2G}{\alpha t}$$ 
We prove the above inequality using mathematical induction. \\
{\bf Base Case ($1\leq t\leq t_q=2L_G^2/\alpha^2+1$)}: As $\x_1, \dots, \x_{t_q}$ are selected randomly, their value doesn't depend on the underlying dataset. Hence, $\x_t=\x_t', \forall 1\leq t\leq t_q$.\\
{\bf Induction Step $t=\tau>2L_G^2/\alpha^2+1$}: We consider two cases:
\begin{itemize}
\item {\bf $F-F'=\{f_\tau\}$}: Since the difference between $F$ and $F'$ is only the $\tau$-th function, hence $\x_\tau=\x_\tau'$. As $\C$ is a convex set, projection onto $\C$ always decreases distance, hence:
\begin{align}
\|\x_{\tau+1}-\x_{\tau+1}'\|_2&\leq \|(\x_\tau-\eta_\tau\grad f_\tau(\x_\tau))-(\x_\tau-\eta_\tau\grad f_\tau'(\x_\tau))\|_2,\nonumber\\
&=\eta_\tau\|\grad f_\tau(\x_\tau)-\grad f_\tau'(\x_\tau)\|_2,\nonumber\\
&\leq \frac{2G}{\alpha \tau}. \nonumber
\end{align}
Hence, lemma holds in this case.
\item {\bf $F-F'=\{f_i\},\ i< \tau$}: Again using convexity of $\C$, we get:
\begin{align}
\|\x_{\tau+1}-\x_{\tau+1}'\|_2^2&\leq \|(\x_\tau-\eta_\tau\grad f_\tau(\x_\tau))-(\x_\tau'-\eta_\tau\grad f_\tau(\x_\tau'))\|_2^2,\nonumber\\
&=\|\x_\tau-\x_\tau'\|_2^2+\eta_\tau^2\|\grad f_\tau(\x_\tau)-\grad f_\tau(\x_\tau')\|^2_2-2\eta_\tau(\x_\tau-\x_\tau')^T(\grad f_\tau(\x_\tau)-\grad f_\tau(\x_\tau')),\nonumber\\
&\leq (1+\eta_\tau^2L_G^2)\|\x_\tau-\x_\tau'\|_2^2-2\eta_\tau(\x_\tau-\x_\tau')^T(\grad f_\tau(\x_\tau)-\grad f_\tau(\x_\tau')),
\label{eq:gigaxxup}
\end{align}
where the last equation follows using Lipschitz continuity of $\grad f_t$. 
Now, using strong convexity:
$$(\x_\tau-\x_\tau')^T(\grad f_\tau(\x_\tau)-\grad f_\tau(\x_\tau'))\geq \alpha\|\x_\tau-\x_\tau'\|_2^2.$$
Combining the above observation and the induction hypothesis with \eqref{eq:gigaxxup}:
\begin{align}
  \|\x_{\tau+1}-\x_{\tau+1}'\|_2^2&\leq \left(1+L_G^2\eta_\tau^2-2\alpha\eta_\tau\right)\cdot\frac{4G^2}{(\tau-1)^2}.
\end{align}
Lemma now follows by setting $\eta_\tau=\frac{2}{\alpha \tau}$ and $\tau>\frac{2L_G^2}{\alpha^2}$.
\end{itemize}
\end{proof}
Using the lemma above with the privacy analysis of \pocp\ (Theorem~\ref{Thm:GenPri}), the privacy guarantee for \pgiga follows immediately.
\begin{thm}[\pgiga\ Privacy]
\pgiga\ (see Algorithm~\ref{alg:giga_diff}) is $(3\epsilon,2\delta)$-differentially private.
\label{Thm:GenPri_giga}
\end{thm}
Next, using the regret bound analysis for \giga\ from \cite{HKKA}(Theorem 1) along with Theorem~\ref{thm:util}, we get the following utility (regret bound) analysis for our \pgiga\ algorithm. Here again, ignoring constants, the regret simplifies to $\tilde{O}(\sqrt{dT})$.
\begin{thm}[\pgiga\ Regret]
\label{thm:util_giga}
Let $\alpha>0$ be the minimum strong convexity parameter of any function $f_t$ in the function sequence $F$. Also, let $L_G$ be the maximum Lipschitz continuity parameter of the gradient of any function $f_t\in F$ and let $G=\max_\tau \|\grad f_t(x)\|_2, \forall x\in \C$. Then, the expected regret of \pgiga\ satisfies
$$\E[\R_{\pgiga}(T)]\leq \frac{4\sqrt{d}(G/\alpha+\|\C\|_2)G\ln^2\frac{T}{\delta}}{\sqrt{\epsilon}}\sqrt{T}+\frac{2G^2}{\alpha}(1+\log T)+\frac{2L_G^2G\ltwo{\mathcal{C}}}{\alpha^2}$$
where $\ltwo{\mathcal{C}}$ is the diameter of the convex set $\mathcal{C}$ and $d$ is the dimensionality of the output space.
\end{thm}
\begin{proof}
Observe that for the first $t_q =\frac{2L_G^2}{\alpha^2}$ iterations \pgiga\ outputs random samples from $\C$. The additional regret incurred during this time is bounded by a constant (w.r.t. T) that appears as the last term in the regret bound given above. For iterations $t \geq t_q$, the proof follows directly by using Theorem~\ref{thm:util} and regret bound of \giga. Note that we use a slightly modified step-size $\eta_t=2/\alpha t$, instead of the standard $\eta_t=1/\alpha t$. This difference in the step size increases the regret of \giga\ as given by \cite{HKKA} by a factor of 2.
\end{proof}

In Section~\ref{sec:IGD} as well this section, we provided examples of the conversion of two standard online learning algorithms into privacy preserving algorithms with provably bounded regret. In both these examples, we show low-sensitivity of the corresponding learning algorithms and use our analysis of \pocp\ to obtain privacy and utility bounds. Similarly, we can obtain privacy preserving variants of many other \ocp\ algorithms such as Follow The Leader (FTL), Follow the Regularized Leader (FTRL) etc. Our low-sensitivity proofs should be of independent interest to the online learning community as well, as they point to a connection between stability (sensitivity) and low-regret (online learnability)---an open problem in the learning community.



\subsection{Logarithmic regret for Quadratic Cost Functions}\noindent
\label{sec:Log}
In Sections~\ref{sec:IGD} and~\ref{sec:GIGA}, we described two differentially private algorithms with $\tilde{O}(\sqrt{T})$ regret for {\it any} strongly convex Lipschitz continuous cost functions. In this section we show that by restricting the cost functions to a practically important class of quadratic functions, we can design a differentially private algorithm to achieve logarithmic regret.

For simplicity of exposition, we consider cost functions of the form:
\begin{equation}\as\bs
f_t(\x) = \frac{1}{2}(y_t - \bm{v}_t^T\x)^2 + \frac{\alpha}{2} \|\x\|^2,
\label{eqn:quadCostF}
\end{equation}
for some $\alpha > 0$. For such cost functions we show that we can achieve $O(\poly(\log{T}))$ regret while providing $(\epsilon, \delta)$-differential privacy.

Our algorithm at a high level is a modified version of the Follow the Leader (\ftl) algorithm \cite{HKKA}. 
 The \ftl\ algorithm obtains the $t$-th step output as:
\begin{equation}\as\bs
\ftl:\qquad\qquad \x_{t+1} = \argmin_{\x \in \C} \sum_{\tau = 1 }^{t}f_{\tau}(\x).
\label{eq:ftl}
\end{equation}
For our quadratic cost function \eqref{eqn:quadCostF} with $\C=\re^d$, the above update yields
\begin{equation}\as\bs
\qftl:\qquad \qquad \x_{t+1}  = (t\alpha\I+ V_{t})^{-1}(\u_t),
\label{eqnref:QuadUpdate}
\end{equation}
where $V_t = V_{t-1}+\v_t\v_t^T$ and $\u_t = \u_{t-1} + y_t \v_t$ with $V_0=0$ and $\u_0=0$. Using elementary linear algebra and assuming $|y_t|\leq R$ and $\|\v_t\|_2\leq R$, we can show that $\|\x_{t+1}\|_2\leq 2R/\alpha, \forall t$. Now, using Theorem 2 of \cite{SK08} along with our bound on $\|\x_t\|_2$, we obtain the following regret bound for the quadratic loss functions based \ftl\ (\qftl) algorithm:
\begin{equation}\as\bs
  \label{eq:reg_qftl}
  \R_{\qftl}(T)\leq \frac{R^4(1+2R/\alpha)^2}{\alpha}\log T.
\end{equation}
Furthermore, we can show that the \qftl\ algorithm (see Equation~\ref{eqnref:QuadUpdate}) also satisfies Assumption~\ref{eq:sens_ocp}. Hence, similar to Sections~\ref{sec:IGD} and~\ref{sec:GIGA}, we can obtain a differentially private variant of \qftl\ with $\tilde{O}(\sqrt{T})$ regret. However, we show that using the special structure of \qftl\ updates (see \eqref{eqnref:QuadUpdate}), we can obtain a differentially private variant of \qftl\ with just $O(\text{poly}(\log T))$ regret, a significant improvement over $\tilde{O}(\sqrt{T})$ regret.

The key observation behind our method is that each \qftl\ update is dependent on the function sequence $F$ through $V_t$ and $\u_t$ only. Hence, computing $V_t$ and $\u_t$ in a differentially private manner would imply differential privacy for our \qftl\ updates as well. Furthermore, each $V_t$ and $\u_t$ themselves are obtained by simply adding an ``update'' to the output at step $t-1$. This special structure of $V_t$ and $\u_t$ facilitates usage of a generalization of the ``tree-based'' technique for computing privacy preserving partial sums proposed by \cite{DNPR}. Note that the ``tree-based'' technique to compute sums (see Algorithm~\ref{Algo:partialSum}) adds significantly lower amount of noise at each step than that is added by our \pocp\ algorithm (see Algorithm~\ref{alg:ocp_diff}). Hence, leading to significantly better regret. Algorithm~\ref{Algo:quad} provides a pseudo-code of our \pqftl\ method. At each step $t$, $\hat{V}_t$ and $\hat{\u}_t$ are computed by perturbing $V_t$ and $\u_t$ (to preserve privacy) using PrivateSum algorithm (see Algorithm~\ref{Algo:partialSum}). Next, $\hat{V}_t$ and $\hat{\u}_t$ are used in the \qftl\ update (see \eqref{eqnref:QuadUpdate}) to obtain the next iterate $\hat{x}_{t+1}$. 
\begin{algorithm}[tb]
	\caption{Private Follow the Leader for Quadratic Cost (\pqftl) }
	\begin{algorithmic}[1]
          \STATE {\bfseries Input:} cost function sequence $F=\langle f_1,\cdots,f_T\rangle$, where each $f_t(x;y_t,\v_t)=(y_t - \bm{v}_t^T\x)^2 + \frac{\alpha}{2} \ltwo{\x}^2$
          \STATE {\bfseries Parameter:} privacy parameters $(\epsilon, \delta)$, $R=\max(\max_t \ltwo{v_t},\max_t |y_t|)$
          \STATE Initialize $\hat\x_1=0^d$
          \STATE Initialize empty binary trees $B^V$ and $B^{\u}$, a data structure to compute $\hat{V}_t$ and $\hat{\u}_t$---differentially private versions of $V_t$ and $\u_t$
          \FOR{$t=1$ to $T-1$}
          \STATE \textbf{Cost:} $\ \ L_t(\hat\x_t)=f_t(\hat\x_t)=(y_t-\v_t^T\hat\x_t)^2+ \frac{\alpha}{2} \ltwo{\hat\x_t}^2$
          \STATE $(\hat{V_t}, B^V)\leftarrow \text{PrivateSum}(\v_t\v_t^T, B^V, t, R^2, \frac{\epsilon}{2}, \frac{\delta}{2}, T)$ (see Algorithm~\ref{Algo:partialSum})
          \STATE $(\hat{\u_t}, B^{\u})\leftarrow \text{PrivateSum}(y_t\v_t, B^{\u}, t, R, \frac{\epsilon}{2}, \frac{\delta}{2}, T)$ (see Algorithm~\ref{Algo:partialSum})
          \STATE \textbf{\qftl\ Update:} $\hat{\x}_{t+1}  \leftarrow (t \alpha \mathbb{I} + \hat{V}_{t})^{-1}(\hat{\u}_{t})$
          \STATE Output $\hat{\x}_{t+1}$
          \ENDFOR
	\end{algorithmic}
	\label{Algo:quad}
\end{algorithm}
Now, we provide both privacy as well as utility (regret bound) guarantees for our \pqftl\ algorithm. First, we prove the privacy of the \pqftl\ algorithm (Algorithm~\ref{Algo:quad}).
\begin{thm}[\pqftl\ Privacy] 
Let $F$ be a sequence of quadratic functions, where $f_t(\x;y_t, \v_t) =\frac{1}{2} (y_t - \v_t^T\x)^2 + \frac{\alpha}{2}\ltwo{\x}^2$. Then, \pqftl\ (Algorithm \ref{Algo:quad}) is $(\epsilon, \delta)$ differential private.
\label{thm:quad_priv}
\end{thm}
\begin{proof}
Using Theorem~\ref{PartialSumPrivacy} (stated in Section \ref{sec:partialSum}), both $\hat{V}_t$ and $\hat{\u}_t$ are each $(  \frac{\epsilon}{2}, \frac{\delta}{2} )$-differentially private w.r.t. $v_t$ and $y_t$, $\forall t$ and hence w.r.t. the function sequence F. Now, $\hat{\x}_{t+1}$ depends on $F$ only through $[\hat{V}_t, \hat{\u}_t]$. Hence, the theorem follows using a standard composition argument \cite{DMNS,DL}.
\end{proof}
Next, we provide regret bound analysis for our \pqftl\ algorithm.
\begin{thm}[\pqftl\ Regret] 
Let $F$ be a sequence of quadratic functions, where $f_t(\x;y_t, \v_t) =\frac{1}{2} (y_t - \v_t^T\x)^2 + \frac{\alpha}{2}\ltwo{\x}^2$. Let $R$ be the maximum $L_2$ norm of any $\v_t$ and $|y_t|$. Then, the regret bound of \pqftl\ (Algorithm \ref{Algo:quad}) satisfies (w.p. $\geq 1-\exp(-d/2)$):
$$\as\bs\R_{\pqftl}(T)=\tilde{O}\left(\frac{\R^6\log \frac{1}{\delta}}{\sqrt{\epsilon}\alpha^3}\sqrt{d}\log^{1.5}T\right).$$
\label{thm:quad_regret}
\end{thm}
\begin{proof}
Using definition of regret,
\begin{align}
\R_{\pqftl} = \sum_{t = 1}^T f_t(\hat{\x}_t) - \argmin_{\x^*}\sum_{t = 1}^T f_t(\x^*)&=  \sum_{t = 1}^T f_t(\x_t) -   \argmin_{\x^*}\sum_{t=1}^Tf_t(\x^*)+\sum_{t = 1}^T (f_t(\hat{\x}_t) -  f_t(\x_t)),\nonumber\\
&\leq \R_{\qftl}(T)+\sum_{t = 1}^T (f_t(\hat{\x}_t) -  f_t(\x_t)),\nonumber\\
&\leq \frac{R^4(1+2R/\alpha)^2}{\alpha}\log T + \sum_{t = 1}^T (f_t(\hat{\x}_t) -  f_t(\x_t)),
\label{eqn:regretQuad}
\end{align}
where last inequality follows using \eqref{eq:reg_qftl}.

Now, as $f_t(\x)$ is a $(R+\alpha)$-Lipschitz continuous gradient function,
\begin{align}
  f_t(\hat{\x}_t)-f_t(\x_t)&\leq ((\v_t^T\x_t-y_t)\v_t+\alpha\x_t)^T(\hat{\x}_t-\x_t)+\frac{R+\alpha}{2}\|\hat{\x}_t-\x_t\|^2,\nonumber\\
&\leq R(2R^2/\alpha+R+2) \|\hat{\x}_t-\x_t\|+\frac{R+\alpha}{2}\|\hat{\x}_t-\x_t\|^2,
  \label{eq:f_lcg}
\end{align}
where last inequality follows using Cauchy-Schwarz inequality and the fact that $\|\x_t\|_2\leq 2R/\alpha$.

We now bound $ \ltwo{\hat{\x}_{t+1} - \x_{t+1} }$. Let $\hat{V}_t = V_t + A_t$ and $\hat{\u}_t = \u_t + \beta_t$  where $A_t$ and $\beta_t$ are the noise additions introduced by the Private Sum algorithm (Algorithm~\ref{Algo:partialSum}).

Now, from the step 9 of \pqftl\ (Algorithm~\ref{Algo:quad}) we have,
\begin{align}
(\hat{V}_{t} + t\alpha \I ) \hat{\x}_{t+1} =  \hat{\u}_t\quad \Leftrightarrow\quad (\frac{1}{t}\hat{V}_{t} + \alpha \I )\hat{\x}_{t+1} &= \frac{1}{t} \hat{\u}_t.  \label{eqnprav1}
\end{align}
Similarly, using \qftl\ update (see \eqref{eqnref:QuadUpdate}) we have,
\begin{align}
(\frac{1}{t} V_t + \alpha \I) \x_{t+1} &= \frac{1}{t}\u_t. \label{eqnprav2}
\end{align}
Using \eqref{eqnprav1} and \eqref{eqnprav2}:
\begin{equation}
 (\frac{1}{t}\hat{V}_{t} + \alpha \I )( \hat{\x}_{t+1} -\x_{t+1}) = \frac{1}{t} \beta_t - \frac{1}{t}A_t\x_{t+1}.
\label{eq:VbA}
\end{equation}
Now, using $\hat{V}_t = V_t + A_t$ and the triangle inequality we have,
\begin{equation}
\ltwo{ (\frac{1}{t}\hat{V}_{t} + \alpha \I )( \hat{\x}_{t+1} -\x_{t+1}) }  \geq \ltwo{ (\frac{1}{t}V_{t} + \alpha \I )( \hat{\x}_{t+1} -\x_{t+1}) } - \ltwo {  \frac{1}{t} A_t( \hat{\x}_{t+1} -\x_{t+1})  }
\label{eq:trineq}
\end{equation}
Furthermore,
 \begin{align}
\ltwo {  \frac{1}{t} A_t( \hat{\x}_{t+1} -\x_{t+1})  } &\leq \frac{1}{t} \ltwo{A_t}\ltwo{\hat{\x}_{t+1} - \x_{t+1} } 
\label{eq:Afbcs}
\end{align}
Thus by combining \eqref{eq:VbA}, \eqref{eq:trineq}, \eqref{eq:Afbcs} and using the fact that the smallest eigenvalue of $(\frac{1}{t}V_{t} + \alpha \I)$ is lower-bounded by $\alpha$,
\begin{equation}
\frac{1}{t}\ltwo{ \beta_t} + \frac{1}{t}\|A_t\|_2\|\x_{t+1}\|_2\geq |\alpha - \frac{\ltwo{A_t}}{t} | \ltwo{\hat{\x}_{t+1} - \x_{t+1} } \label{eq:Afbcs3}
\end{equation}
Now using Theorem~\ref{PartialSumPrivacy} each entry of the matrix $A_t$ is drawn from $\mathcal{N}( 0, \sigma^2\log T)$ for $\sigma^2 = \frac{R^2}{\epsilon}\log^2{T}\log{\frac{\log{T}}{\delta}}$. Thus the spectral norm of $A_t$, $\ltwo{A_t}$ is bounded by $3\sigma\sqrt{d}$ with probability at least $1-\exp(-d/2)$. 
 Similarly, $\ltwo{\beta_t} \leq 3\sigma \sqrt{d}$, with probability at least $1-\exp(-d/2)$. Also, $\ltwo{\x_t} \leq 2R/\alpha$. Using the above observation with \eqref{eq:Afbcs3},
 \begin{align}
 \ltwo{\hat{\x}_{t+1} - \x_{t+1} } \leq  \frac{\sigma\sqrt{d}}{t}\cdot\frac{3+6R/\alpha}{|\alpha - \frac{6\sigma\sqrt{d}R}{\alpha t}|}.
\label{eq:Afbcs4}
 \end{align}
Using \eqref{eqn:regretQuad}, \eqref{eq:f_lcg}, and \eqref{eq:Afbcs4}, we get (with probability at least $1-\exp(-d/2)$):
\begin{equation}
  \label{eq:regretQuad1}
  \R_{\pqftl}(T)\leq \frac{R^4(1+2R/\alpha)^2}{\alpha}\log T+3\sqrt{d}R(2R^2/\alpha+R+2)(1+2R/\alpha)(1+\log T)\frac{1}{\sqrt{\epsilon}} \sqrt{\log{T}} \log{\sqrt{\frac{log{T}}{\delta}}}.
\end{equation}
Hence w.h.p., $$\R_{\pqftl}(T)=\tilde{O}\left(\frac{\R^6\log \frac{1}{\delta}}{\sqrt{\epsilon}\alpha^3}\sqrt{d}\log^{1.5}T\right).$$
\end{proof}
\subsubsection{ Computing Partial Sums Privately }\noindent
\label{sec:partialSum}
In this section, we consider the problem of computing partial sums while preserving differential privacy. Formally, let  $D=\langle \w_{1},\w_{2},\cdots,\w_T\rangle$ be a sequence of vectors, where at each time step $t$, a vector $\w_t\in \re^d$ is provided. Now the goal is to output partial sums $W_t=\sum_{\tau=1}^t \w_\tau$ at each time step $t$, without compromising privacy of the data vectors in $D$. Note that by treating a matrix as a long vector obtained by row concatenation, we can use the same approach to compute partial sums over matrices as well.

Now, note that $L_2$-sensitivity of each partial sum is $O(R)$ ($R=\max_t \|\w_t\|_2$), as changing one $w_\tau$ can change a partial sum by an additive factor of $2R$. Hence, a na\"{\i}ve method is to add $O(R\sqrt{\frac{\log\frac{1}{\delta}}{\epsilon}})$ noise at $t$-th to obtain $(\epsilon,\delta)$-privacy for a fixed step $t$. Using standard composition argument, overall privacy of such a scheme over $T$ iterations would be $(T\epsilon, T\delta)$. Hence, to get a constant $(\epsilon',\delta')$ privacy, we would need to add $O(R\sqrt{T}\sqrt{\frac{\log\frac{T}{\delta'}}{\epsilon'}})$ noise. In contrast, our method, which is based on a generalization of \cite{DNPR}, is able to provide the same level of privacy by adding only $O(R\log T\sqrt{\frac{\log\frac{\log T}{\delta'}}{\epsilon'}})$ noise. 
We first provide a high level description of the algorithm and then provide a detailed privacy and utility analysis.
\begin{figure}[t]
\begin{tabular}{cc}
    \includegraphics[width=3in]{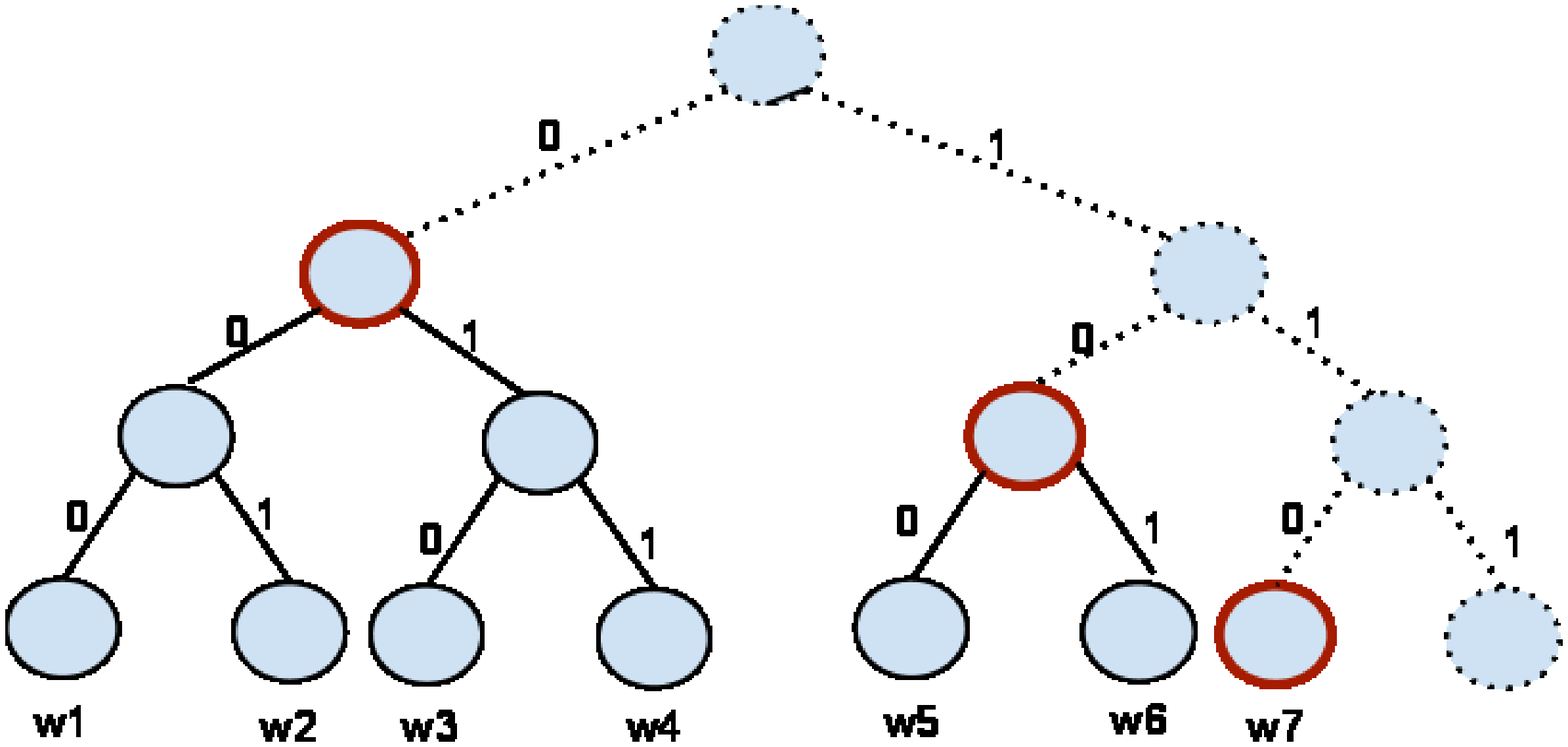}&
    \hspace*{30pt}\includegraphics[width=3in]{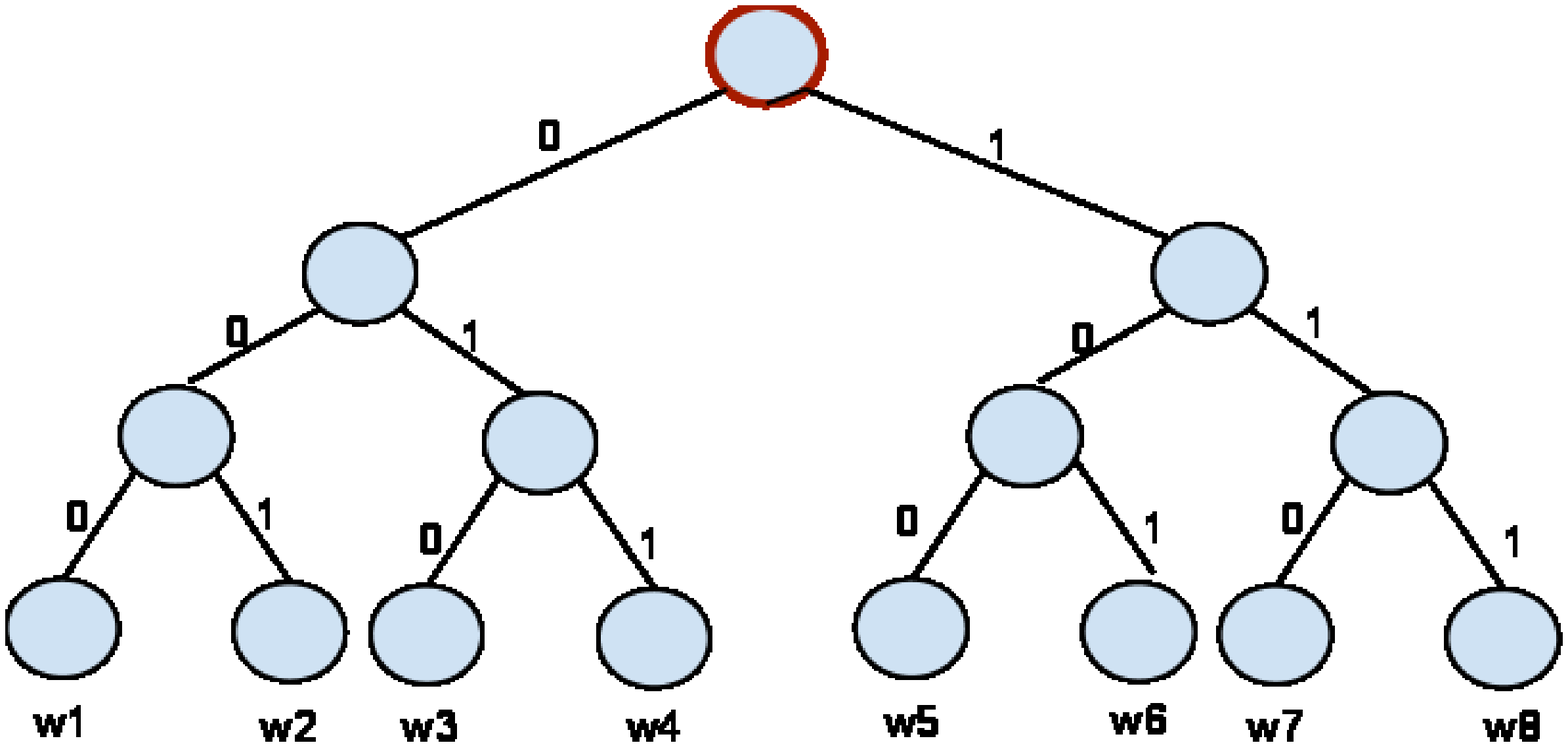}\\
(a)& (b)
\end{tabular}
\caption{Binary Tree for $T=8$. Each node in the tree has noise drawn from $\mathcal{N}(0,\sigma^2\I^d)$ including the leaves. The edge labels on the path from root to any node form the label for that node. {\bf (a)}: $\w_1, \w_2,..,\w_7$ are the input vectors that have arrived till time step $t= 7$. Each internal node is obtained by adding noise from $\mathcal{N}(0, \sigma^2\I^d)$ to the sum of input vectors in the sub-tree rooted at the node. To return the partial sum at $t=7$, return the sum of the nodes in thick red. The dotted nodes are unpopulated. {\bf (b)}: The figure depicts the change in the data structure after the arrival of $\w_8$. Now the partial sum at $t=8$ is obtained by using just one node denoted in thick red.}
\label{fig:logSum}
\end{figure}

\newcommand{\hB}{\hat{\B}}
Following \cite{DNPR}, we first create a binary tree $\B$ where each leaf node corresponds to an input vector in $D$. We denote a node at level $i$ (root being at level $0$) with strings in $\{0,1\}^{i}$ in the following way: For a given node in level $i$ with label $s\in\{0,1\}^i$, the left child of $s$ is denoted with the label $s\circ 0$ and the right child is denoted with $s\circ 1$. Here the operator $\circ$ denotes concatenation of strings. Also, the root is labeled with the empty string $\empty$.

Now, each node $s$ in the tree $\B$ contains two values: $\B_s$ and $\hB_s$, where $\B_s$ is obtained by the summation of vectors in each of the leaves of the sub-tree rooted at $s$, i.e., $\B_s = \sum_{\begin{subarray}{c}j: j  = s\circ r\\ r \in \{0,1\}^{k-i}\end{subarray} } \wu_j$. Also, $\hB_s=\B_s+\b_s$ is a perturbation of $\B_s$, $\b_s\sim {\mathcal N}(0, \sigma^2\I^d)$, and $\sigma$ is as given in Lemma~\ref{tech-lemma}.


A node in the tree is populated only when all the vectors that form the leaves of the sub-tree rooted at the node have arrived. Hence, at time instant $t$ we receive vector $\w_t$ and populate the nodes in the tree $\B$ for which all the leaves in the sub-tree rooted at them have arrived. To populate a node labeled $s$, we compute $\B_s = \B_{s\circ0} + \B_{s\circ1}$, the sum of the corresponding values at its two children in the tree and also $\hat{\B}_s=\B_s+\b_s$,  $\b_s\sim {\mathcal N}(0, \sigma^2\I^d)$. 

As we prove below in Lemma~\ref{tech-lemma}, for a $i$-th level node which is populated and has label $s \in \{0,1\}^i$, $\hat{\B}_s$ contains an $(\epsilon, \delta)$-private sum of the $2^{k-i}$ vectors that correspond to the leaves of the sub-tree rooted at $s$. Now, to output a differentially private partial sum at time step $t$, we add up the perturbed values at the highest possible nodes that can be used to compute the sum. Note, that such a summation would have at most one node at each level. See Figure \ref{fig:logSum} for an illustration. We provide a pseudo-code of our method in Algorithm \ref{Algo:partialSum}.

\begin{algorithm}[tb]
	\caption{Private Sum($\w_t, \B, t, R, \epsilon, \delta, T$)}
	\begin{algorithmic}[1]
		\REQUIRE Data vector $\w_t$, current binary tree $\B$, current vector number $t$, $R$ a bound on $\ltwo{\w_t}$, privacy parameters $\epsilon$ and $\delta$, total number of vectors $T$, dimensionality of vectors $d$
                \IF{$t=1$}
                \STATE Initialize the binary tree $\B$ over $T$ leaves with all nodes
                \STATE $ \sigma^2 \leftarrow \frac{R^2}{\epsilon} \log^2{T}\log{\frac{\log{T }}{\delta}}$
                \ENDIF
		\STATE  $s_{t}\leftarrow $ the string representation of $t$ in binary
		\STATE  $\B_{s_t} \leftarrow \w_t$\hspace{1cm}    //Populate the $s_t$-th entry of $\B$
		\STATE $\hat{\B}_{s_t} \leftarrow \B_{s_t} + \b_{s_t}$, where $\b_{s_t} \sim \mathcal{N}(0, \sigma^2\I^d)$
                \STATE Let $S_t$ is the set of all ancestors $s$ of $s_t$ in the tree $\B$, such that all the leaves in the sub-tree rooted at $s$ are already populated
                \FORALL{$ s\in S_t$}
                \STATE $\B_s \leftarrow \B_{s\circ 0} + \B_{s\circ 1} \qquad \qquad //$ $B_s$ is the value at node with label $s$ (without noise)
                \STATE $\hat{\B}_s \leftarrow \B_s + \b_s$, where $\b_s \sim \mathcal{N}(0, \sigma^2\I^d)$ $\qquad \qquad //$ $\hat{B}_s$ is the noisy value at node with label $s$
                \ENDFOR
                \STATE Find the minimum set of \emph{already} populated nodes in $\B$ that can compute $\sum_{\tau=1}^t\w_\tau$. Formally, starting from the left, for each bit position $i$ in $s_t$ such that $s_t(i) = 1$, form strings $s^q = s_t(1)\circ...\circ s_t(i-1)\circ0$ of length $i$. Let $s^1, s^2,... ,s^Q$ be all such strings, where $Q \leq \log {T}$. For example, if $s_t = 110$ then the strings obtained this way are: $0$ and $10$
	\STATE {\bf Output:} $(\hat W_t=\sum_{q=1}^Q \hat{\B}_{s^q}, \B)$ \label{line:addNoise}
	\end{algorithmic}
	\label{Algo:partialSum}
\end{algorithm}

Theorem \ref{PartialSumPrivacy} states privacy as well as utility guarantees of our partial sums method (Algorithm~\ref{Algo:partialSum}). We first provide a technical lemma which we later use in our proof of Theorem~\ref{PartialSumPrivacy}.

Let $\hat{\B}(D)$ denote the set of all perturbed node values $\hB_s, \forall s$ obtained by applying Algorithm~\ref{Algo:partialSum} on dataset $D$. Also, $D$ and $D'$ be two datasets that differ in at most one entry, say $\w_t$.
 \begin{lemma}
Let $\hat\B_s(D)=\B_s(D)+\b_s$, where $\b_s\sim\mathcal{N}(0,\sigma^2\I^d)$ for  $ \sigma^2 = \frac{R^2}{\epsilon} \log^2{T}\log{\frac{\log{T }}{\delta}}$. Then, 
 for any $t$ and any $\Theta_s \in \mathbb{R}^d$,
$$\as\bs \pdf [ \hat{\B}_{s}(D) = \Theta_s] \leq  e^{\frac{\epsilon}{\log{T}} } \pdf[ \hat{\B}_{s}(D') = \Theta_s]+\frac{\delta}{\log{T}}$$
where $D$ and $D'$ are two datasets differing in exactly one entry.
\label{tech-lemma}
\end{lemma}
\begin{proof}
Let $\Delta = \B_s(D) - \B_s(D')$. Note that $\|\Delta\|_2\leq R$. Now, consider the following ratio:
\begin{align}
 \frac{ \pdf [ \hat{\B}_{s}(D) = \Theta_s]}  { \pdf [\hat{\B}_{s}(D') = \Theta_s]} &= \frac { \exp{\frac{\ltwo{\Theta_s - \B_s(D)  }^2}{2\sigma^2}}}{ \exp{\frac{\ltwo{\Theta_s - \B_s(D')  }^2}{2\sigma^2}} }
= \exp{ \frac{ \ltwo{\Delta }^2  - 2\Delta ^T (\B_s(D')- \Theta_s) }{2\sigma^2} },\nonumber\\
&\leq \exp{ \frac{ R^2  + 2|\Delta ^T (\B_s(D')- \Theta_s)| }{2\sigma^2} }.
\label{eq:Bspr}
\end{align}
Now, $\Delta^T(\B_s(D')- \Theta_s)$ follows $\mathcal{N}(0,\ltwo{\Delta}^2\sigma^2 )$.  
For a random variable $V\sim\mathcal{N}(0,1)$, and for all $\gamma>1$,  $\pdf[|V|>\gamma]\leq e^{-\gamma^2/2}$ ( Mill's inequality ). Thus,
  $$\pdf [| \Delta^T(\B_s(D') - \Theta_s)| \geq R \sigma \gamma]\leq \pdf [| \Delta^T(\B_s(D') - \Theta_s)| \geq\ltwo{\Delta} \sigma \gamma] \leq \exp(\frac{-\gamma^2}{2})$$
Lemma follows by setting $ \gamma= 2\sqrt{   \ln{   \frac{\log{T}}{\delta}} }$ in the equation above and combining it with \eqref{eq:Bspr}.
\end{proof}
Next, we provide formal privacy and utility guarantees for Algorithm \ref{Algo:partialSum}. Our proof is inspired by a technique developed by \cite{DNPR}.
\begin{thm}[Algorithm~\ref{Algo:partialSum}: Privacy and Utility] \label{PartialSumPrivacy}
Let $D=\langle\w_1,\cdots,\w_T\rangle$ be a dataset of vectors with $\w_t\in\re^d$ being provided online at each time step $t$. Let $R= \max_{i \leq T}\ltwo{\wu_i} $ and $ \sigma^2 = \frac{R^2}{\epsilon} \log^2{T}\log{\frac{\log{T }}{\delta}}$. Let $W_t=\sum_{\tau=1}^t\w_\tau$ be the partial sum of the entries in the dataset $D$ till the $t$-th entry. Then, $\forall t\in[T]$, following are true for the output of Algorithm \ref{Algo:partialSum} with parameters $(t,\epsilon, \delta, R, T)$.
\begin{itemize}\setlength{\topsep}{-1pt}\setlength{\itemsep}{1pt}
    \item \textbf{Privacy:} The output $\hat W_t$ is $(\epsilon,\delta)$-differentially private.
    \item \textbf{Utility:} The output $\hat W_t$ has the following distribution: $\hat W_t\sim\mathcal{N}(W_t,k\sigma^2\I_d)$, where $k\leq \lceil\log T\rceil.$
\end{itemize}
\end{thm}
\begin{proof}
{\bf Utility}: Note that Line \ref{line:addNoise} of the Algorithm~\ref{Algo:partialSum} adds at most $\lceil\log{T}\rceil$ vectors $\hB_s$ (corresponding to the chosen nodes of the binary tree $\B$). Now each of the selected vectors $\hB_s$ is generated by adding a noise $\b_s\sim \mathcal{N}(0, \sigma^2\I^d)$. Furthermore, each $\b_s$ is generated independent of other noise vectors. Hence, the total noise in the output partial sum $\hat W_t$ has the following distribution: $\mathcal{N}(0,k\sigma^2\I_d)$, where $k\leq \lceil\log{T}\rceil$.\\ 

\textbf{Privacy:} 
First, we prove that $\hat{\B}(D)$ is $(\epsilon, \delta)$-differentially private. As defined above, let $D$ and $D'$ be the two  datasets (sequences of input vectors) that differ in exactly one entry. Let $S \subset \mathbb{R}^{2T-1}$. Now, $$\frac{\Pr [ \hat{\B}(D) \in  S]}{\Pr [\hat{\B}(D') \in S]} =  \frac{\int_{\Theta \in S}\pdf [ \hat{\B}(D) = \Theta]}{\int_{\Theta \in S} \pdf [\hat{\B}(D') =\Theta]}.$$
Note that noise ($\b_s$) at each node $s$ is generated independently of all the other nodes. Hence,
$$\frac{ \pdf [ \hat{\B}(D) = \Theta]}  { \pdf [\hat{\B}(D') = \Theta]}  =  \frac{ \Pi_{s}\pdf [ \hat{\B}_{s}(D) = \Theta_{s}]}{ \Pi_{s}\pdf [ \hat{\B}_{s}(D') = \Theta_{s}]}. $$
Since $D$ and $D'$ differ in exactly one entry, $B(D)$ and $B(D')$ can differ in at most $\log{T}$ nodes. Thus at most $\log{T}$ ratios in the above product can be different from one. Now, by using Lemma \ref{tech-lemma} to bound each of these ratios and then using composability argument \cite{DMNS,DL} over the $\log{T}$ nodes which have differing values in $\B(D)$ and $\B(D')$,
$$  \Pr [ \hat{\B}(D) = \Theta] \leq e^{\epsilon} \Pr [\hat{\B}(D') \in \Theta] +\delta,$$
i.e., $\hat{\B}(D)$ is $(\epsilon, \delta)$-differentially private.

Now, each partial sum is just a deterministic function of $\hat{\B}(D)$. Hence, $(\epsilon, \delta)$-differential privacy of each partial sum follows directly by $(\epsilon, \delta)$-differential privacy of $\hat{\B}(D)$. %
\end{proof}
\section{Discussion}
\subsection{Other Differentially Private Algorithms}\noindent
Recall that in Section~\ref{sec:IGD}, we described our Private \igd\ algorithm that achieves $\tilde O(\sqrt{T})$ regret for any sequence of strongly convex, Lipschitz continuous functions. While, this  class of functions is reasonably broad, we can further drop the strong convexity condition as well, albeit with higher regret.  To this end, we perturb each $f_t$ and apply \igd\ over $\tilde{f}_t=f_t+\frac{\alpha}{\sqrt{t}}\ltwo{\x-\x_0}$, where $\x_0$ is randomly picked point from the convex set $\C$. We can then show that under this perturbation ``trick'' we can obtain sub-linear regret of $\tilde O(T^{3/4})$. The analysis is similar to our analysis for \igd\ and requires a fairly straightforward modification of the regret analysis by \cite{KB10}.

We now briefly discuss our observations about the Exponentially Weighted Online Optimization (EWOO) \cite{HAS}, another \ocp\ algorithm with sub-linear regret bound. This algorithm does not directly fit into our Private \ocp\ framework, and is  not wide-spread in practice due to relatively inefficient updates (see \cite{HAS} for a detailed discussion). However, just for completeness, we note that by using techniques similar to our Private \ocp\ framework and using \emph{exponential mechanism} (see \cite{MT}), one can analyze this algorithm as well to guarantee differential privacy along with $\tilde O(\sqrt T)$ regret.
\subsection{Application to Offline Learning}\noindent
\label{sec:offlineLearn}
In Section~\ref{sec:OCP}, we proposed a generic online learning framework that can be used to obtain differentially private online learning algorithms with good regret bounds. Recently, \cite{SKT} showed that online learning algorithms with good regret bounds can be used to solve several offline learning problems as well. In this section, we exploit this connection to provide a generic differentially private framework for a large class of offline learning problems as well.

In a related work, \cite{CM} also proposed a method to obtain differentially private algorithms for offline learning problems. However, as discussed later in the section, our method covers a wider range of learning problems, is more practical and obtains better error bounds for the same level of privacy. 

First, we describe the standard offline learning model that we use. In typical offline learning scenarios, one receives (or observes) a set of {\em training} points sampled from some fixed distribution and also a loss function parametrized by some hidden parameters. Now, the goal is to learn the hidden parameters such that the expected loss over the same distribution is minimized.

Formally, consider a domain $\mathcal{Z}$ and an arbitrary distribution $\D_\mathcal{Z}$ over $\mathcal{Z}$ from which training data is generated. Let $D=\langle \z_1,\cdots,\z_T\rangle$ be a training dataset, where each $\z_i$ is drawn i.i.d. from the distribution $\D_\mathcal{Z}$. Also, consider a loss function $\ell:\C\times\mathcal{Z}\to\re^+$, where $\C\subseteq\re^d$ be a (potentially unbounded) convex set. Let $\ell(\cdot;\cdot)$ be a convex function, $L$-Lipschitz in both the parameters and let $\ell(0;\z)\leq 1, \forall \z\in \mathcal{Z}$. Intuitively, the loss function specifies goodness of a learned model $\x\in \C$ w.r.t. to the training data.  Hence, the goal is to solve the following minimization problem (also called Risk Minimization):
\begin{equation}\as\bs
  \label{eq:rm}
  \min_{\x\in\mathcal{C}} \E_{\z\sim\D_\mathcal{Z}}[\ell(\x;\z)].
\end{equation}
Let $\x^*$ be the optimal solution to \eqref{eq:rm}, i.e., $\x^*=\arg\min_{\x\in\mathcal{C}} \E_{\z\sim\D_\mathcal{Z}}[\ell(\x;\z)].$.
Recently, \cite{SKT} provided an algorithm to obtain an additive approximation to \eqref{eq:rm} via online convex programming (\ocp). The algorithm of \cite{SKT} is as follows: execute any reasonable \ocp\ algorithm $\A$ (like \igd\ or \giga\,) on the function sequence $F=\langle\ell(\x;\z_1)+\frac{\alpha}{2}\|\x\|^2,\cdots,\ell(\x;\z_T)+\frac{\alpha}{2}\|\x\|^2\rangle$ in an online fashion. Furthermore, if the set $\C$ is an unbounded set, then it can be set to be an $L_2$ ball of radius $\|\x^*\|_2$, i.e,
$$\as\bs\C=\{\x: \x\in \re^d, \|\x\|_2\leq \|\x^*\|_2\}.$$
Now, let $\x_1,\cdots,\x_T$ be the sequence of outputs produced by $\A$. Then, output $\tilde \x=\frac{1}{T}\sum_{t=1}^T \x_t$ as an approximation for $\x^*$. Theorem~\ref{thm:skt} bounds additional error incurred by $\tilde\x$ in comparison to $\x^*$. Next, to produce differentially private output we can add appropriate noise to the output $\tilde{\x}$. We present a detailed pseudo-code in Algorithm~\ref{alg:stoc_diff}. For simplicity of presentation, we instantiate our framework with the \igd\ algorithm as the underlying \ocp\ algorithm.
\begin{algorithm}[tb]
	\caption{Private Offline Learning (\pol)}
	\begin{algorithmic}[1]
		\STATE {\bfseries Input:} Input dataset $D=\langle \z_1,\cdots,\z_T\rangle$ and the convex set $\C$
                \STATE {\bfseries Parameter:} Privacy parameters $(\epsilon_p, \delta)$, generalization error parameter $\epsilon_g$, Lipschitz bound $L$ on the loss function $\ell$, bound on $\|\x^*\|_2$
                \STATE If $\C=\R^d$ then set $\C=\{\x: \x\in \R^d, \|\x\|_2\leq \|\x^*\|_2\}.$
		\STATE Choose $\x_1$ randomly from $\C$
                \STATE Set $\alpha\leftarrow \frac{\epsilon_g}{\|\x^*\|_2^2}$
		\STATE Initialize $\bm s= \x_1$
        \FOR{$t=1$ to $T-1$}
                \STATE \textbf{Learning rate:} $\eta_t=\frac{1}{\alpha t}$
                \STATE \textbf{IGD Update:} $\ \ \x_{t+1}\leftarrow \argmin_{\x\in \C} \left(\frac{1}{2} \|\x-\x_t\|_2^2+\eta_t (\ell(\x;\z_t)+\frac{\alpha}{2}\|\x\|_2^2)\right)$
                \STATE \textbf{Store sum:} $\bm s\leftarrow\bm s+\x_{t+1}$
        \ENDFOR
        \STATE \textbf{Average:} $\tilde\x\leftarrow \frac{\bm s}{T}$
        \STATE \textbf{Noise Addition:} $\bar\x\leftarrow \tilde\x+\b$, where $\b\sim\mathcal{N}(0^d,\beta^2\I^d )$ and $\beta=\frac{2\sqrt{2}(L+\alpha\|\x^*\|_2)\ln T}{T\epsilon_p}\sqrt{\ln\frac{1}{\delta}+\epsilon_p}$
        \STATE Output $\hat\x=\argmin_{\x\in \C}\left(\|\x-\bar\x\|_2^2\right)$					
	\end{algorithmic}
	\label{alg:stoc_diff}
\end{algorithm}

First, we show that \pol\ (Algorithm \ref{alg:stoc_diff}) is $(\epsilon,\delta)$-differentially private.
\begin{thm}[\pol\ Privacy]
The Private Offline Learning (\pol) algorithm (see Algorithm~\ref{alg:stoc_diff}) is $(\epsilon_p,\delta)$-differentially private.
\label{Thm:GenPri_off}
\end{thm}
\begin{proof}
Recall that to prove differential privacy, one needs to show that changing one training points from the dataset $D$ will not lead to significant changes in our algorithm's output $\hat{\x}$ which is a perturbation of $\tilde{\x}=\frac{1}{T}\sum_{t=1}^T \x_t$. Hence, we need to show that the $L_2$-sensitivity (see Definition~\ref{defn:sensitivity}) of $\tilde{\x}$ is low.

Now let $\x'_1,\cdots,\x'_T$ be the sequence of outputs produced by the \igd\ algorithm used in Algorithm \ref{alg:stoc_diff} when executed on a dataset $D'$ which differs in exactly one entry from $D$. To estimate the sensitivity of $\tilde \x$, we need to bound $\ltwo{\frac{1}{T}\sum_{t=1}^T (\x_t-\x'_t)}.$ Now, using triangle inequality and Lemma~\ref{lemma:sens_igd}, we get:
\begin{equation}\as\bs
  \label{eq:sens_stoc}
  \ltwo{\frac{1}{T}\sum_{t=1}^T (\x_t-\x'_t)}\leq \frac{1}{T}\sum_{t=1}^T\|\x_t-\x'_t\|_2\leq \frac{1}{T}\sum_{t=2}^T\frac{2 L'}{t-1}\leq \frac{2L\ln T}{T},
\end{equation}
where $L'$ is the maximum Lipschitz continuity coefficient of $\ell(\x,\z_t)+\frac{\alpha}{2}\|\x\|_2^2, \forall t$ over the set $\C$. Using the fact that $\|\C\|_2=\|\x^*\|_2$, we obtain $L'=L+\alpha\|\x^*\|_2$.

The theorem now follows using $L_2$-sensitivity of $\tilde{\x}$ (see \eqref{eq:sens_stoc}) and a proof similar to that of Lemma \ref{Lem:SingGenPriv}.
\end{proof}

With the privacy guarantee in place, we now focus on the utility of Algorithm \ref{alg:stoc_diff}, i.e., approximation error for the Risk Minimization problem \eqref{eq:rm}. We first rewrite the approximation error incurred by $\tilde{\x}=\frac{1}{T}\sum_{t=1}^T\x_t$, as derived by \cite{SKT}.
\begin{thm}[Approximation Error in Risk Minimization (Eq. \ref{eq:rm}) \cite{SKT}]
Let $\R_\A(T)$ be the regret for the online algorithm $\A$. Then with probability at least $1-\gamma$,
$$\as\bs\E_{z\sim\D_\mathcal{Z}}[\ell(\tilde\x;z)]-\E_{z\sim\D_\mathcal{Z}}[\ell(\x^*;z)]\leq \frac{\alpha}{2}\|\x^*\|^2+\frac{\R_\A(T)}{T}+\frac{4}{T}\sqrt{\frac{L'^2\R_\A(T)\ln(\frac{4\ln T}{\gamma})}{\alpha}}+\frac{\max\{\frac{16L'^2}{\alpha},6\}\ln(\frac{4\ln T}{\gamma})}{T}$$
where $L'=L+\alpha \|\x^*\|_2$, $L$ is the Lipschitz continuity bound on the loss function $\ell$ and $\alpha$ is the strong convexity parameter of the function sequence $F$.
\label{thm:skt}
\end{thm}

\begin{thm}[\pol\ Utility (Approximation Error in Eq. \ref{eq:rm})]
Let $L$ is the Lipschitz bound on the loss function $\ell$ and $T$ be the total number of points in the training dataset $D=\{\z_1, \dots, \z_T\}$. Let $(\epsilon_p, \delta)$ be differential privacy parameters, and $d$ be the dimensionality. Then, with probability at least $1-\gamma$,
$$\as\bs\E_{z\sim\D_\mathcal{Z}}[\ell(\hat\x;z)]-\min_{\x\in\C}\E_{z\sim\D_\mathcal{Z}}[\ell(\x;z)]\leq \epsilon_g,$$
when the number of points sampled ($T$) follows,
$$\as\bs T\geq C\max \left(\frac{\sqrt{d}L(L+\epsilon_g/\|\x^*\|_2)\sqrt{\ln\frac{1}{\gamma}\ln\frac{1}{\delta}}}{\epsilon_g\epsilon_p},\frac{(L+\epsilon_g/\|\x^*\|_2)^2\|\x^*\|_2^2\ln T \ln\frac{\ln T}{\gamma}}{\epsilon_g^2}\right),$$
where $C>0$ is a  global constant.
%
\label{thm:util_stoc}
\end{thm}
\begin{proof}
To prove the result, we upper bound $\E_{z\sim\D_\mathcal{Z}}[\ell(\hat\x;z)]-\E_{z\sim\D_\mathcal{Z}}[\ell(\x^*;z)]$ as:
\begin{align}
\E_{z\sim\D_\mathcal{Z}}[\ell(\hat\x;z)]-\E_{z\sim\D_\mathcal{Z}}[\ell(\x^*;z)]&=\E_{z\sim\D_\mathcal{Z}}[\ell(\hat\x;z)]-\E_{z\sim\D_\mathcal{Z}}[\ell(\tilde\x;z)]+\E_{z\sim\D_\mathcal{Z}}[\ell(\tilde\x;z)]-\E_{z\sim\D_\mathcal{Z}}[\ell(\x^*;z)],\nonumber\\
&\leq L\ltwo{\hat{\x}-\tilde{\x}}+\E_{z\sim\D_\mathcal{Z}}[\ell(\tilde\x;z)-\ell(\x^*;z)],\nonumber\\
&= L\ltwo{\b}+\E_{z\sim\D_\mathcal{Z}}[\ell(\tilde\x;z)-\ell(\x^*;z)],
\label{eq:geup}
\end{align}
where the second inequality follows using Lipschitz continuity of $\ell$ and the last equality follows by the noise addition step (Step 13) of Algorithm~\ref{alg:stoc_diff}.

From the tail bound on the norm of Gaussian random vector, it follows that with probability at least $1-\frac{\gamma}{2}$,
\begin{equation}
\ltwo{\b}\leq 3\sqrt{d}\beta\sqrt{\ln\frac{1}{\gamma}}\leq 12\sqrt{d}L'\frac{\ln T}{T\epsilon_p}\sqrt{\ln\frac{1}{\gamma}\ln\frac{1}{\delta}},
\label{eq:bup}
\end{equation}
where  $L'=L+\epsilon_g/\|\x^*\|_2$, $L$ is the Lipschitz continuity parameter of $\ell$. Note that in Line 5 of Algorithm \ref{alg:stoc_diff} we set the strong convexity parameter $\alpha=\frac{\epsilon_g}{\ltwo{\x^*}^2}$.

Now, regret bound of \igd\ is given by:
\begin{equation}
\label{eq:igd_reg1}
R_{\igd}(T)=O(\epsilon_g+\frac{L'}{\alpha}\ln T),
\end{equation}
Thus, by combining \eqref{eq:geup}, \eqref{eq:bup}, \eqref{eq:igd_reg1}, and Theorem~\ref{thm:skt}, with probability at least $1-\gamma$,
$$\E_{z\sim\D_\mathcal{Z}}[\ell(\hat\x;z)]-\min_{\x\in\C}\E_{z\sim\D_\mathcal{Z}}[\ell(\x;z)]\leq \frac{\epsilon_g}{2}+C\frac{\sqrt{d}L(L+\frac{\epsilon_g}{\|\x^*\|_2})\ln T\sqrt{\ln\frac{1}{\gamma}\ln\frac{1}{\delta}}}{\epsilon_p T}+C\frac{(L+\frac{\epsilon_g}{\|\x^*\|_2})^2\|\x^*\|_2^2\ln T\ln \frac{\ln T}{\gamma}}{\epsilon_g T},$$
where $C>0$ is a global constant.

The result now follows by bounding the RHS above by $\epsilon_g$.
\end{proof}
We note that although our Algorithm~\ref{alg:stoc_diff} and analysis assumes that the underlying \ocp\ algorithm is \igd, however our algorithm and analysis can be easily adapted to use with any other \ocp\ algorithm by plugging in the regret bound and $L_2$ sensitivity of the corresponding \ocp\ algorithm.\\
{\bf Comparison to existing differential private offline learning methods:} Recently, \cite{CM} proposed  two differentially private frameworks for a wide range of offline learning problems, namely, \emph{output perturbation} and \emph{objective perturbation}. However, our method has three significant advantages over both the methods of \cite{CM}:
\begin{itemize}\setlength{\topsep}{-10pt}
\setlength{\itemsep}{-5pt}
\item {\bf Handles larger class of learning problems}: Note that both privacy analysis (Theorem~\ref{Thm:GenPri_off}) and utility analysis (Theorem~\ref{thm:util_stoc}) only require the loss function $\ell$ to be a convex, Lipschitz continuous function. In fact, the loss function is not required to be even differentiable. Hence, our method can handle {\em hinge loss}, a popular loss function used by Support Vector Machine (SVM). In comparison, \cite{CM} requires the loss function $\ell$ to be {\em twice} differentiable and furthermore, the gradient should be Lipschitz continuous. \\
Furthermore, our method can be used for minimizing risk (see \eqref{eq:rm}) over any fixed convex constraint set $\C$. In contrast, \cite{CM} requires the  set $\C$ to be the complete vector space $\re^d$.
\item {\bf Better error bound}: Theorem 18 of \cite{CM} bounds the sample size by $T=O(\frac{\|\x^*\|^2_2\ln \frac{1}{\delta}}{\epsilon_g^2}+\frac{d\|\x^*\|_2}{\epsilon_g\epsilon_p})$, which is same as our bound (see Theorem~\ref{thm:util_stoc}) except for an additional $\sqrt{d}$ factor. Hence, our analysis provides tighter error bound w.r.t. dimensionality of the space. We believe the difference is primarily due to our usage of Gaussian noise instead of Gamma noise added by \cite{CM}. 
\item {\bf More practical}: Our method provides an explicit iterative method for solving \eqref{eq:rm} and hence provides differential privacy guarantees even if the algorithm stops at any step $T$. In contrast, \cite{CM} assumes optimal solution to a certain optimization problem, and it is not clear how the differential privacy guarantees of \cite{CM} extends when the optimization algorithm is forced to halt prematurely and hence might not give the optimal solution. 
\end{itemize}
In a related work, \cite{RBHT} also proposed a differentially private framework for offline learning. However, \cite{RBHT} compares the point-wise convergence of the obtained solution $\hat{\x}$ to the private optimum of true risk minimizer $\x^*$, where as \cite{CM} and our method (see Algorithm~\ref{alg:stoc_diff}) compare the approximation error; hence, results of \cite{RBHT} are incomparable to our results. 

\section{Empirical Results}\noindent
\label{sec:EmpRel}
In this section we study the privacy and utility (regret) trade-offs for two of our private \ocp\ approaches under different practical settings. Specifically, we consider the practically important problem of online linear regression and online logistic regression. For online linear regression we apply our \pqftl\ approach (see Algorithm~\ref{Algo:quad}) and for online logistic regression we apply our \pigd\  method (see Algorithm~\ref{alg:igd}). For both the problems, we compare our method against the offline optimal and the non-private online version and show the regret/accuracy trade-off with privacy. We show that our methods learn a meaningful hypothesis (a hyperplane for both the problems) while privacy is provably preserved due to our differential privacy guarantees.
\subsection{Online Linear Regression (OLR)}
Online linear regression (OLR) requires solving for $\x_t$ at each step so that squared error in the prediction is minimized. Specifically, we need to find $\x_t$ in an online fashion such that $\sum_t (y_t-g_t^T\x_t)^2 + \alpha\|\x_t\|^2$ is minimized. OLR is a practically important learning problem and have a variety of practical applications in domains such as finance \cite{KW}.

Now, note that we can directly apply our \pqftl\ approach (see Section~\ref{sec:Log}) to this problem to obtain differentially private iterates $\x_t$ with the regret guaranteed to be logarithmic. Here, we apply our \pqftl\ algorithm for the OLR problem on a synthetic dataset as well as a benchmark real-world dataset, namely ``Year Prediction'' \cite{uci}. For the synthetic dataset, we fix $\x^*$, generate data points $g_t$ of dimensionality $d=10$ by sampling a multivariate Gaussian distribution and obtain the target $y_t=g_t^T\x^*+\eta$, where $\eta$ is random Gaussian noise with standard variance $0.01$. We generate $T=100,000$ such input points and targets. The Year Prediction dataset is $90$-dimensional and contains around $500,000$ data points. For both the datasets, we set $\alpha=1$ and at each step apply our \pqftl\ algorithm. We measure the optimal offline solution using standard ridge regression and also compute regret obtained by the non-private \ftl\ algorithm.

Figure~\ref{fig:logreg} (a) and (b) shows the average regret(\ie regret normalized by the number of entries $T$) incurred by \pqftl\ for different privacy level $\epsilon$ on synthetic and Year Prediction data. Note that the y-axis is on the log-scale. Clearly, our \pqftl\ algorithm obtains low-regret even for reasonable high privacy levels ($\epsilon=0.01$). Furthermore, the regret gets closer to the regret obtained by the non-private algorithm as privacy requirements are made weaker.

\begin{figure}
\centering
\hspace*{-30pt}
\begin{tabular}{ccc}
\begin{minipage}{.33\textwidth}
\includegraphics[width=\columnwidth]{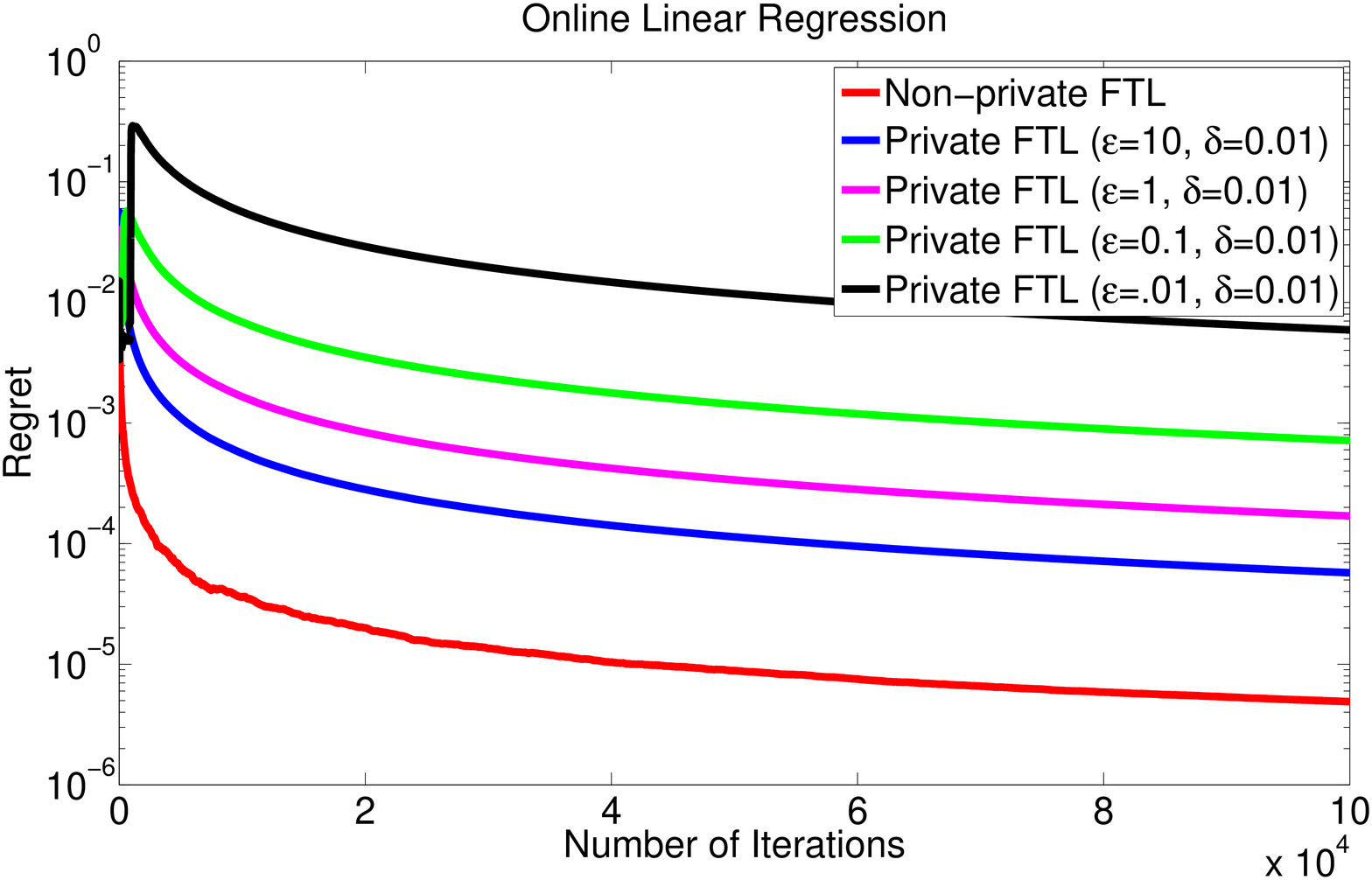}
\end{minipage}&
\hspace*{-10pt}
\begin{minipage}{.33\textwidth}
\includegraphics[width=\columnwidth]{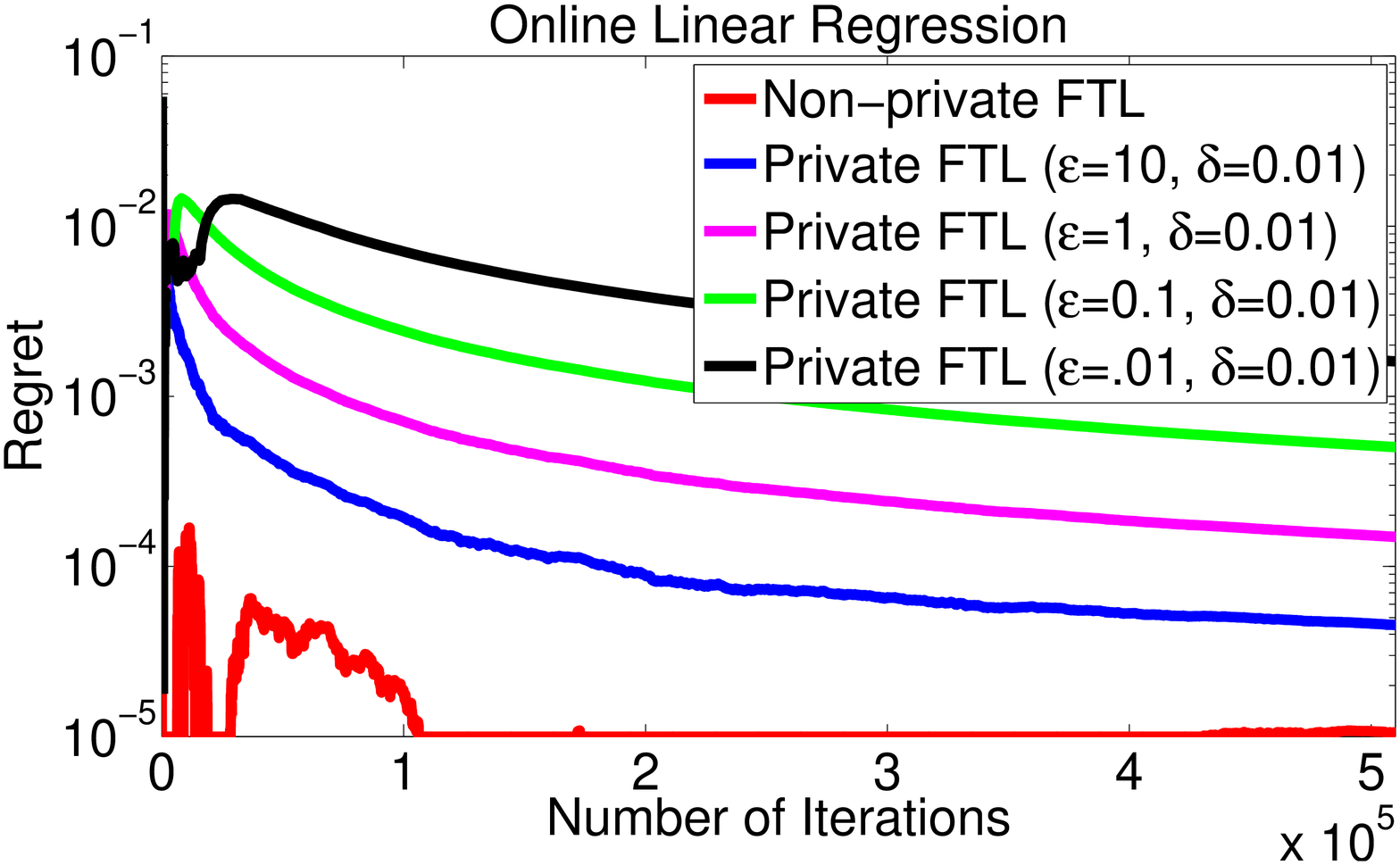}
\end{minipage}&
\hspace*{-5pt}
\begin{minipage}{.33\textwidth}
\centering
{\small
\begin{tabular}{|c|c|}\hline
Method&Accuracy\\ \hline
Non-private \igd&68.1\%\\ \hline
\pigd\ ($\epsilon=20, \delta=0.01$)&66.3\%\\ \hline
\pigd\ ($\epsilon=10, \delta=0.01$)&62.7\%\\ \hline
\pigd\ ($\epsilon=1, \delta=0.01$)&59.4\%\\ \hline
\pigd\ ($\epsilon=0.1, \delta=0.01$)&58.3\%\\ \hline
\end{tabular}}
\end{minipage}\\
(a)&(b)&(c)
\end{tabular}
\caption{Privacy vs Regret. {\bf (a), (b)}: Average regret (normalized by the number of iterations) incurred by \ftl\ and \pqftl\ with different levels of privacy $\epsilon$ on the synthetic $10$-dimensional data and Year Prediction Data. Note that the regret is plotted on a log-scale. \pqftl\ obtained regret of the order of $1e-2$ even with high privacy level of $\epsilon=0.01$. {\bf (c)}: Classification accuracy obtained by \igd\  and \pigd\ algorithm on Forest-covertype dataset. \pigd\ learns a meaningful classifier while providing privacy guarantees, especially for low privacy levels, i.e., high $\epsilon$.}
\label{fig:logreg}%
\end{figure}
\subsection{Online Logistic Regression}
Online logistic regression is a variant of the online linear regression where the cost function is logistic loss rather than squared error. Logistic regression is a popular method to learn classifiers, and has been shown to be successful for many practical problems. In this experiment, we apply our private \igd algorithm to the online logistic regression problem. To this end, we use the standard Forest cover-type dataset, a dataset with two classes, $54$-dimensional feature vectors and $581,012$ data points. We select $10\%$ data points for testing purpose and run our Private \igd\ algorithm on the remaining data points. Figure~\ref{fig:logreg} (c) shows classification accuracy (averaged over $10$ runs) obtained by \igd\ and our \pigd\ algorithm for different privacy levels. Clearly, our algorithm is able to learn a reasonable classifier from the dataset in a private manner. Note that our regret bound for \pigd\ method is $O(\sqrt{T})$, hence, it would require more data points to reduce regret to very small values, which is reflected by a drop in classification accuracy as $\epsilon$ decreases.


\section{Conclusions}\noindent
In this paper, we considered the problem of differentially private online learning. We used online convex programming (\ocp) as the underlying online learning model and described a method to achieve sub-linear regret for the \ocp\ problem, while maintaining $(\epsilon, \delta)$-differential privacy of the data (input functions). Specifically, given an arbitrary \ocp\ algorithm, we showed how to produce  a private version of the algorithm and proved the privacy guarantees by bounding the sensitivity of the algorithm's output at each step $t$. We considered two well known algorithms (\igd\ and \giga\,) in our framework and provided a private version of each of the algorithm. Both of our differentially private algorithms have $\tilde{O}(\sqrt{T})$ regret while guaranteeing $(\epsilon, \delta)$ differential privacy. We also showed that for the special case of quadratic cost functions, we can obtain logarithmic regret while providing differential privacy guarantees on the input data. Finally, we showed that our differentially private online learning approach can be used to obtain differentially private algorithms for a large class of convex offline learning problems as well. Our approach can handle a larger class of offline problems and obtains better error bounds than the existing methods \cite{CM}.   

While we can provide logarithmic regret for the special class of quadratic functions, our regret for general strongly convex functions is $\tilde O(\sqrt{T})$. An open question is if the $\tilde O(\sqrt{T})$ bound that we obtain is optimal or if it can be further improved. Similarly, another important open question is to develop privacy preserving techniques for the \ocp\ problem that have a poly-logarithmic dependence on the dimension of the data. Finally, another interesting research direction is an extension of our differentially private framework from the ``full information'' \ocp\ setting to the bandit setting.

\paragraph{Acknowledgments.} We would like to thank Ankan Saha, Adam Smith, Piyush Srivastava and Ambuj Tewari for various intriguing conversations during the course of this project.
\bibliographystyle{plain}
\bibliography{references}
\normalsize
\end{document}